\documentclass{article}

\PassOptionsToPackage{numbers,compress}{natbib}
 \usepackage[preprint]{neurips_2026}

\usepackage[utf8]{inputenc} 
\usepackage[T1]{fontenc}    
\usepackage{hyperref}       
\usepackage{url}            
\usepackage{booktabs}       
\usepackage{amsfonts}       
\usepackage{nicefrac}       
\usepackage{microtype}      
\usepackage{xcolor}         
\usepackage{xspace}
\usepackage{amsmath,amssymb,amsthm}
\usepackage{cleveref}
\usepackage{graphicx}
\usepackage{adjustbox}
\usepackage{bm}
\usepackage{tikz}
\usepackage{subcaption}
\usepackage{caption}
\usepackage{wrapfig}
\usepackage{titletoc}

\theoremstyle{plain} 
\newtheorem{theorem}{Theorem}

\newtheorem{corollary}[theorem]{Corollary}

\theoremstyle{definition} 

\newtheorem{assumption}[theorem]{Assumption}

\theoremstyle{remark} 

\newcommand{\ie}{\textit{i.e.}}
\newcommand{\eg}{\textit{e.g.}}
\newcommand{\wrt}{\textit{w.r.t.}\xspace}
\newcommand{\pr}{\mathbb{P}}
\newcommand{\doint}{\text{do}}
\newcommand{\bE}{\mathbb{E}}
\newcommand{\KL}{\mathrm{KL}}

\newcommand{\npc}{NPC\xspace}
\newcommand{\cnpc}{CNPC\xspace}
\newcommand{\pc}{PC\xspace}

\title{Causal Neural Probabilistic Circuits}

\author{%
  Weixin Chen \\
  University of Illinois Urbana-Champaign\\
  \texttt{weixinc2@illinois.edu} \\
  \And
  Han Zhao\\
  University of Illinois Urbana-Champaign\\
  \texttt{hanzhao@illinois.edu} \\
}

\begin{document}

\maketitle

\begin{abstract}
Concept Bottleneck Models (CBMs) enhance the interpretability of end-to-end neural networks by introducing a layer of concepts and predicting the class label from the concept predictions. 
A key property of CBMs is that they support interventions, \ie, domain experts can correct mispredicted concept values at test time to improve the final accuracy. 
However, typical CBMs apply interventions by overwriting only the corrected concept while leaving other concept predictions unchanged, which ignores causal dependencies among concepts.
To address this, we propose the Causal Neural Probabilistic Circuit (\cnpc), which combines a neural attribute predictor with a causal probabilistic circuit compiled from a causal graph. This circuit supports exact, tractable causal inference that inherently respects causal dependencies.
Under interventions, \cnpc models the class distribution based on a Product of Experts (PoE) that fuses the attribute predictor’s predictive distribution with the interventional marginals computed by the circuit.
We theoretically characterize the compositional interventional error of \cnpc \wrt its modules and identify conditions under which \cnpc closely matches the ground-truth interventional class distribution. Experiments on five benchmark datasets in both in-distribution and out-of-distribution settings show that, compared with five baseline models, \cnpc achieves higher task accuracy across different numbers of intervened attributes.
\end{abstract}

\section{Introduction} \label{sec:intro}
Despite achieving high performance across a range of benchmarks~\citep{imagenet, glue}, deep learning models are often black-box and lack interpretability~\citep{stop_explain}, which makes it difficult to understand why a specific decision is made, let alone to interact with the decision process. Concept Bottleneck Models (CBMs)~\citep{cbm, nesy_concepts} offer a promising direction by introducing a concept layer that decomposes the decision process into two stages: a neural concept predictor first infers the values of a predefined set of human-understandable concepts from the input, and a label predictor, often a linear layer, then maps the predicted concept values to the final class label. This architecture provides concept-level explanations and, more importantly, supports interventions, \ie, a human can correct mispredicted concept values. Such interventions are particularly valuable in safety-critical applications such as medical diagnosis~\citep{medical, clinical}, where a domain expert can inspect and correct concept values to improve the accuracy and reliability of the final prediction.

In practice, however, typical CBMs~\citep{cbm, cem, label_free_cbm} implement interventions by overwriting only the corrected concept while leaving all other concept predictions unchanged; see Figure~\ref{fig:framework} (top panel). It ignores the causal dependencies among concepts, where intervening on one concept could lead to updates on other concepts that depend on it~\citep{pearl2009causal, peters2017causal}. For example, knowing that a patient smokes should increase the likelihood of lung cancer. Therefore, to improve intervention efficiency and make better use of costly expert interventions, CBMs should propagate interventions through the causal structure over concepts and the class label, yielding the interventional class distribution for updated prediction.

\begin{figure*}[t]
    \centering
    \setlength{\abovecaptionskip}{4pt}
    \includegraphics[width=\linewidth]{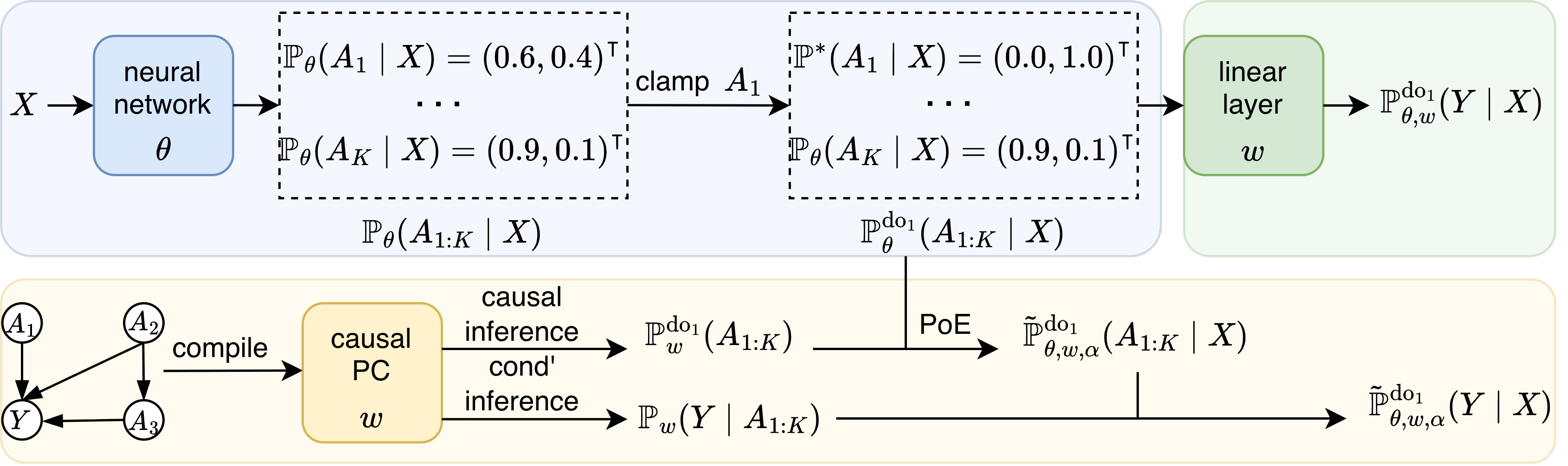}
    \caption{\small 
    Model architectures and intervention procedures of a CBM and a \cnpc.
    A typical CBM (top-left module + top-right module) performs the intervention on $A_1$ by replacing the neural network's predictions for $A_1$ with the ground-truth distribution, while leaving the predictions for other attributes unchanged.
    In contrast, \cnpc (top-left module + bottom module) combines a neural network with a causal \pc compiled from a causal graph, and approximates the interventional class distribution based on a PoE with a balancing weight $\alpha$ that fuses complementary information from the two modules.
    }
    \label{fig:framework}
    \vspace{-8pt}
\end{figure*}

In this paper, however, we show that exactly modeling this interventional class distribution is challenging in CBMs: the interventional attribute\footnote{In the CBM literature, a \textit{concept} is defined as a binary variable indicating the presence or absence of a human-interpretable property (\eg, \textit{is\_red\_color}, \textit{is\_yellow\_color}), whereas we use a more general term \textit{attribute}, defined as a categorical variable that takes values from a finite set of such properties (\eg, \textit{color}).} distribution cannot, in general, be obtained using the standard CBM modules.
To address this challenge, inspired by the recent Neural Probabilistic Circuit (\npc)~\citep{npc, npc_workshop}, we propose the Causal Neural Probabilistic Circuit (\cnpc).
An \npc consists of a neural attribute predictor and a probabilistic circuit (\pc)~\citep{spn, pc_family, vergari2020probabilistic} that performs exact, tractable \textit{probabilistic inference} over attributes and the class label.
It predicts the class distribution by combining the predicted attribute distribution and the conditional distribution of the class label given attributes computed by the \pc.
We extend \npc to \cnpc by additionally compiling a causal graph over attributes and the class label into the \pc, yielding a causal \pc~\citep{causalpc} that supports exact, tractable \textit{probabilistic inference} and \textit{causal inference}.
When there is no intervention, \cnpc uses the same prediction form as \npc.
Under interventions on one or more attributes, we approximate the interventional attribute distribution using a Product of Experts (PoE)~\citep{poe} that combines the neural predictor’s attribute distribution with the interventional marginal attribute distribution computed by the causal \pc. \cnpc then predicts the class distribution by combining this PoE-based approximation with the conditional distribution of the class label given attributes computed by the causal \pc.

Our theoretical results show that both \npc and \cnpc exhibit compositional errors \wrt the neural predictor and the (causal) \pc under interventions;
moreover, under certain conditions, \cnpc matches the ground-truth interventional class distribution more closely than \npc.
To empirically evaluate \cnpc, we consider four settings: one benign (in-distribution) setting and three out-of-distribution (OOD) settings where the test-time input distribution deviates from the training distribution. The OOD settings include unseen data transformations~\citep{hendrycks2019benchmarking}, adversarial perturbations~\citep{goodfellow2015adversarial}, and spurious-correlation shifts~\citep{arjovsky2019invariant}. In these OOD settings, the neural predictor may perform unreliably, thus necessitating interventions.
Empirical results for \cnpc and five baseline models across five benchmark datasets show three main findings. 
First, task accuracy increases for all models as the number of intervened attributes grows, demonstrating the effectiveness of interventions. 
Second, in the benign setting, all models perform similarly, with \cnpc achieving a small advantage on most datasets.
Third, in the OOD settings, \cnpc substantially outperforms all baselines across different numbers of intervened attributes, indicating its high efficiency in utilizing interventions.

Our contributions are summarized as follows.
\textbf{(1)} We propose \cnpc, a model that combines a neural attribute predictor with a causal \pc compiled from the causal graph over attributes and the class label. \cnpc uses a PoE to approximate the interventional class distribution, which accounts for the propagation of attribute interventions.
\textbf{(2)} We provide theoretical results characterizing the compositional interventional errors of \npc and \cnpc, and show that \cnpc can model the ground-truth interventional class distribution more accurately than \npc under certain conditions.
\textbf{(3)} We empirically show that \cnpc achieves higher intervention efficiency compared with various baseline models, especially in OOD settings.

\section{Related work} \label{sec:related}
\paragraph{Modeling concept dependencies in CBMs.}
Traditional CBMs~\citep{cbm} predict each concept independently from the input, which neglects dependencies among concepts.
Modeling such dependencies can be useful, especially under intervention, as one can update the values of remaining concepts given an intervened concept.
A representative approach is SCBM~\citep{scbm}, which models values of all concepts with a multivariate normal distribution and trains two neural networks to predict the mean vector and covariance matrix from the input, respectively.
The interventional concept distribution is derived by conditioning the multivariate normal on the intervened concepts.
Similarly, \citet{leakage} leverage an auto-regressive model to capture concept dependencies.
%
Recent work \citep{c2bm, cgm} considers the causal dependencies among concepts and the class variable, explicitly incorporating a ground-truth or learned causal structure by arranging these variables as nodes of a causal graph and training neural networks to approximate the outcomes of the underlying structural equations. With the same goal, we instead compile the causal graph into a \pc, which performs exact and tractable causal inference whose outputs are guaranteed to respect the causal semantics of the graph by construction.

\paragraph{Intervention strategies.}
Due to the high cost of human interventions, the number of concepts that can be intervened on is often limited.
This naturally raises the question of which concepts to select under a fixed intervention budget.
Ideally, an intervention strategy should choose the concepts that yield the largest improvement in downstream task performance.
\citet{closer_look, sheth2022learning} study a range of selection criteria (\eg, prediction uncertainty) as well as different selection granularities (\eg, intervening on individual concepts versus concept groups).
They find that no single criterion is consistently best across settings, and that intervening on concept groups can be cost-effective when concepts within a group are mutually exclusive.
Henceforth, interventions are performed by default on concept groups, \ie, attributes.
Beyond these hand-crafted selection criteria, \citet{interactive, learn_interactive} learn policies that adaptively choose the next concept to intervene on.
With the causal structure explicitly incorporated, \citet{c2bm} select attributes in a top-down order along the causal hierarchy, as intervening on higher-level attributes can potentially influence more downstream attributes.
In this work, we follow them and select attributes for intervention by their depth in the causal graph.
However, it remains unclear whether this strategy is the most intervention-efficient for \cnpc. We leave it as a future direction and discuss it in Appendix~\ref{sec:discuss}.

\section{Preliminaries} \label{sec:prelim}

\subsection{Notation and assumptions} \label{subsec:assump}
Throughout the paper, random variables, their instantiations, and the instantiation sets are denoted by capital, lowercase, and calligraphic letters, respectively, \eg, $X, x, \mathcal{X}$.
CBMs predict the class label solely from the predicted attribute values, which relies on the following assumption:

\begin{assumption}[Sufficient attributes~\citep{npc}] \label{assump:sufficient}
    The class label $Y$ and the input $X$ are conditionally independent given the attributes $A_{1:K}$, \ie, $Y \perp X \mid A_{1:K}$.
\end{assumption}

In the case of insufficient attributes, one may leverage large language models (LLMs) or vision-language models (VLMs) for automated concept discovery and labeling~\citep{label_free_cbm, posthoc_cbm, vlg_cbm}, which remains an active research problem in the CBM literature. While this is complementary to our method, we assume that sufficient attributes are annotated, as our main focus is the incorporation of causal~dependencies.

We also assume access to the structure\footnote{We assume access only to the graph structure, while the specific causal mechanisms can remain unknown.} of a causal graph $\mathcal{C}$ over the semantically meaningful, observable variables, namely the attributes $A_{1:K}$ and the class label $Y$ in our setting. Each node in the graph represents a variable, and each directed edge $V_1 \to V_2$ indicates that $V_1$ is a causal parent (direct cause) of $V_2$. 
The availability of the graph structure is common in the causal inference literature~\citep{pearl1995causal, pearl2009causal, pearl2002causal} and is typically obtained from domain knowledge. When it is unknown, one may apply existing causal discovery algorithms~\citep{causal_discovery_pc, causal_discovery_hillclimb, causal_discovery_ges} to learn a plausible graph, which is complementary to our method.

We assume there are no unobserved confounders among the variables in $\mathcal{C}$. For a random variable $V \in \mathcal{C}$, we use $\mathrm{Pa}_{\mathcal{C}}(V)$ to denote the parent nodes of $V$ in $\mathcal{C}$ and $\mathrm{ND}_{\mathcal{C}}(V)$ to denote the non-descendants of $V$ in $\mathcal{C}$. In this work, we make the following structural assumption on $\mathcal{C}$:
\begin{assumption}[Structural assumption on the causal graph]
\label{assump:structure}
    $\mathrm{Pa}_{\mathcal{C}}(Y) \subseteq \{A_{1:K}\} \subseteq \mathrm{ND}_{\mathcal{C}}(Y)$.
\end{assumption}

This assumption states that every causal parent of $Y$ in $\mathcal{C}$ is an attribute, and every attribute is a non-descendant of $Y$. It is natural in CBM settings, as it allows attributes to influence the class~label, directly or indirectly through other attributes, while excluding causal influence in the reverse direction.

\subsection{Neural probabilistic circuits}
A Neural Probabilistic Circuit (\npc)~\citep{npc,chenunderstanding} consists of two modules. 
The first is an \textit{attribute predictor}, which is a neural network with $K$ classification heads and parameterized by $\theta$. It takes an input $X$ and predicts $K$ attribute distributions $\left\{ \pr_\theta(A_k\mid X) \right\}_{k=1}^K$.
The second is a \textit{probabilistic circuit} (\pc)~\citep{pc_family}, specifically a sum-product network~\citep{spn}, parameterized by $w$. It represents the joint distribution $\pr_w(Y, A_{1:K})$ while supporting the tractable computation of the conditional distribution $\pr_w(Y \mid A_{1:K})$. 
Leveraging these two modules, \npc predicts the class distribution as:
{\small
\begin{equation}
\label{eq:npc}
\begin{aligned}
\pr_{\theta, w}(Y\mid X) &= \sum_{a_{1:K}} \pr_{w}(Y \mid A_{1:K} = a_{1:K})\cdot\prod_{k=1}^K \pr_{\theta}(A_k=a_k \mid X).
\end{aligned}
\end{equation}}

Although not explicitly described in the original paper, the interventional distribution modeled by \npc can be inferred from the standard intervention procedure used in CBMs.
Consider an intervention on attribute $A_j$, denoted by $\doint_j$.
\npc would replace the $j$-th predicted attribute distribution with the ground-truth distribution $\pr_*(A_j\mid X)$\footnote{In practice, $\pr_*(A_j\mid X)$ is typically $\mathbb{I}(A_j=a_j^*(X))$ for some $a_j^*(X)\in\mathcal{A}_j$, which we adopt throughout; relaxing it to stochastic attribute labelings is a natural extension but would not change the qualitative claims of this paper.} while leaving the remaining distributions unchanged.
Consequently, \npc models the interventional class distribution under intervention on $A_j$ as:
{\small
\begin{equation}\label{eq:npc_int}
\begin{aligned}
&\pr^{\doint_j}_{\theta, w}(Y \mid X)
= \sum_{a_{1:K}} \pr_w(Y \mid A_{1:K}=a_{1:K}) \cdot \pr^{\doint_j}_\theta(A_{1:K}=a_{1:K} \mid X), \\
&\text{where }\pr^{\doint_j}_\theta(A_{1:K}=a_{1:K} \mid X)
= \pr_*(A_j=a_j \mid X) \cdot \prod_{k\neq j} \pr_\theta(A_k=a_k \mid X).
\end{aligned}
\end{equation}}
This formulation naturally extends to interventions on multiple attributes.
Similar to typical CBMs, \npc does not account for causal dependencies among attributes during intervention. One potential barrier is that standard \pc{}s do not support tractable causal inference.

\subsection{Causal probabilistic circuits} \label{sec:cpc}

Traditional approaches to interventional queries reduce the problem to exact inference on the mutilated graph induced by the intervention (the causal graph with edges into intervened variables removed~\citep{pearl1995causal}) with per-query complexity exponential in treewidth~\citep{dechter1999bucket, koller2009probabilistic, darwiche2009modeling}. In contrast, \citet{compile_graph1, causalpc} show that a causal graph can be compiled into a \pc that, once built, supports exact causal inference in time linear in circuit size.\footnote{We clarify that the circuit size itself can be exponential in the worst case, reflecting the inherent hardness of exact inference rather than a limitation of compilation.} 

In particular, they use the Variable Elimination (VE) algorithm~\citep{koller2009probabilistic} to schedule factor operations within the network polynomial~\citep{network_polynomial}, and construct an arithmetic circuit accordingly.
As an example, consider three variables with the graph structure $V_2 \leftarrow V_1 \rightarrow V_3$.
Each variable corresponds to a factor that represents its conditional probability distribution (CPD): $\phi_{V_1}(V_1)$, $\phi_{V_2}(V_1, V_2)$, and $\phi_{V_3}(V_1, V_3)$, whose symbolic entries are $\phi_{v_1}$, $\phi_{v_2\mid v_1}$, and $\phi_{v_3\mid v_1}$, respectively.
Each variable is also associated with an evidence factor, $\lambda_{V_i}(V_i)$, with entries given by indicator functions $\lambda_{v_i}=\mathbb{I}(V_i=v_i)$.
With these notations, the network polynomial is expressed as:
{\small
\begin{equation*}
\begin{aligned}
f(\bm{\phi}, \bm{\lambda}) &= \sum_{v_1, v_2, v_3} \phi_{v_1} \phi_{v_2\mid v_1} \phi_{v_3\mid v_1} \lambda_{v_1} \lambda_{v_2} \lambda_{v_3}.
\end{aligned}
\end{equation*}}

Given an elimination order, \eg, $[V_2, V_3, V_1]$, VE schedules the factor operations (multiplication and sum-out), yielding an equivalent expression that is more efficient to evaluate~\citep{dechter1999bucket, zhang1996causal}:
{\small
\begin{equation*}
\begin{aligned}
f(\bm{\phi}, \bm{\lambda}) &= \sum_{v_1} \phi_{v_1} \lambda_{v_1} \cdot \left( \sum_{v_3}  \phi_{v_3\mid v_1} \lambda_{v_3} \right) \cdot \left( \sum_{v_2} \phi_{v_2\mid v_1} \lambda_{v_2} \right).
\end{aligned}
\end{equation*}}
See the full derivation in Appendix~\ref{app:cpc}.

\begin{wrapfigure}{r}{0.65\textwidth}
    \centering
    \setlength{\abovecaptionskip}{4pt}
    \setlength{\belowcaptionskip}{-10pt}
    \includegraphics[width=\linewidth]{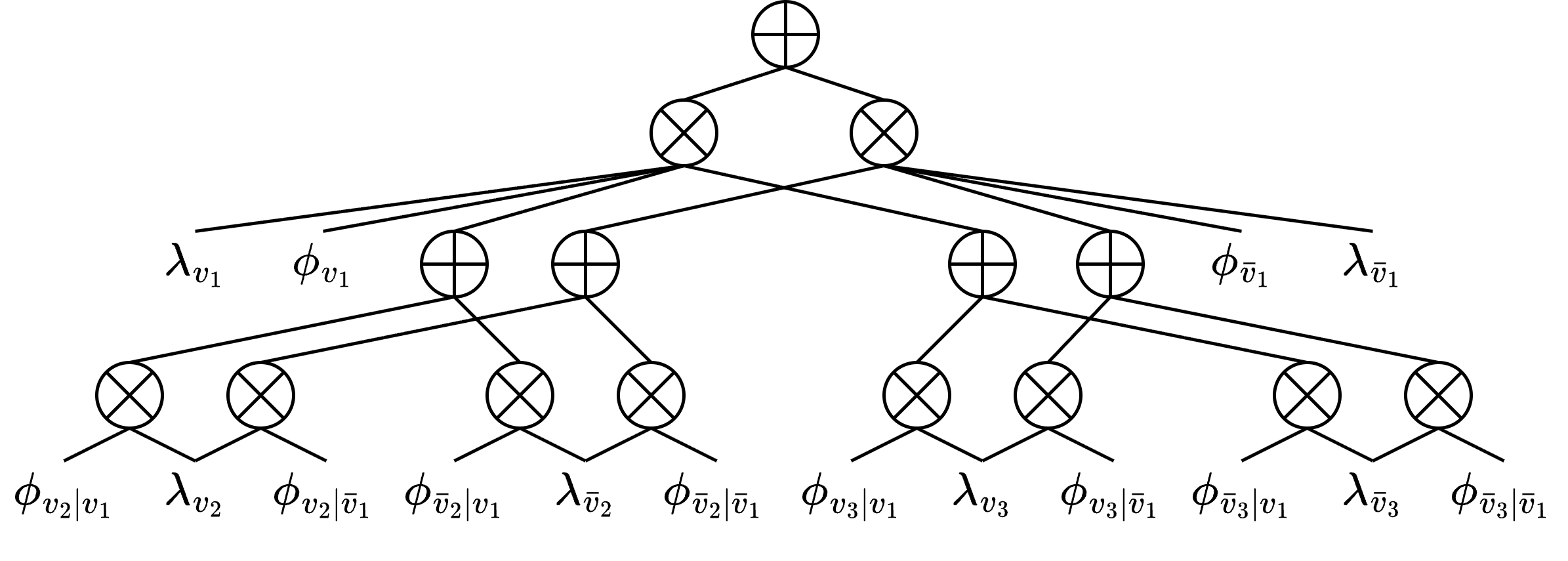}
    \caption{\small A causal \pc that compiles the causal graph $V_2 \leftarrow V_1 \rightarrow V_3$.}
    \label{fig:causal_pc}
\end{wrapfigure}


An arithmetic circuit is then constructed by establishing a sum node for the sum-out operator while establishing a product node for the multiplication operator.
For simplicity, suppose all variables are binary\footnote{The extension to the categorical ones
is standard.}, and we use $v_i$ and $\bar{v}_i$ to denote $V_i=1$ and $V_i=0$.
The resulting circuit is illustrated in Figure~\ref{fig:causal_pc}.
We refer to this circuit that compiles the causal graph as the \textit{causal \pc}.
For the values of those symbolic entries, one can use gradient descent or EM to seek maximum-likelihood estimates~\citep{causalpc}.

As the causal \pc encodes the network polynomial, it supports tractable probabilistic inference~\citep{network_polynomial}.
Computing the probability of any joint or marginal event, \eg, $\Pr(v_1, \bar{v}_2)$, only requires setting the indicators whose subscript is compatible with the event to $1$, \ie, $\lambda_{v_1}=1$, $\lambda_{\bar{v}_2}=1$, $\lambda_{v_3}=1$, $\lambda_{\bar{v}_3}=1$, and enabling one forward pass in \pc. Conditional inference, hence, requires two \pc forward passes.
With the compilation of the causal graph, the causal \pc also supports tractable causal inference.
Computing the interventional probability\footnote{The causal \pc also supports tractable inference for counterfactual queries. As it is out of the scope of this paper, we do not discuss the details here.}, \eg, $\Pr^{\doint(\bar{v}_2)}(v_1)$, can be done in two steps:
(i)~Replace the CPD entries of the intervened variable with $1$, \ie, $\phi_{v_2\mid v_1} = 1$, $\phi_{v_2\mid \bar{v}_1} = 1$, $\phi_{\bar{v}_2\mid v_1} = 1$, $\phi_{\bar{v}_2\mid \bar{v}_1} = 1$;
(ii)~Set the indicators compatible with $(v_1, \bar{v}_2)$ to $1$.
Following these two steps, the causal \pc outputs $\phi_{v_1}$, \ie, $\Pr(v_1)$, which is the expected result as $V_2$ has no causal effect on $V_1$.

\section{Method} \label{sec:method}

\subsection{Why is modeling the interventional class distribution hard} \label{sec:hardness}

Given an input $X$, the class distribution induced by the intervention $\doint_j$, which sets $A_j$ to a fixed value according to $\pr_*(A_j\mid X)$, can be written as:
{\small
\begin{equation}
\label{equ:trueint}
\begin{aligned}[b]
    \pr^{\doint_j}(Y\mid X) 
    &= \sum_{a_{1:K}} \pr(Y\mid A_{1:K}=a_{1:K}) \cdot \pr^{\doint_j}(A_{1:K}=a_{1:K}\mid X).
\end{aligned}
\end{equation}}
This expression follows from Assumption~\ref{assump:sufficient}, Assumption~\ref{assump:structure} together with the local Markov property, and the invariance of causal mechanisms. The full derivation is provided in Appendix~\ref{app:derivation}.


The first term in Eq.~\eqref{equ:trueint} can be evaluated exactly and tractably by a \pc or a causal \pc;
the second term, however, is not directly available from any module.
Using the Bayes' rule, we can rewrite it as:
{\small
\begin{equation*}
\begin{aligned}
    \pr^{\doint_j}(A_{1:K}=a_{1:K}\mid X) &\propto \pr^{\doint_j}(A_{1:K}=a_{1:K})\cdot \pr^{\doint_j}(X\mid A_{1:K}=a_{1:K}).
\end{aligned}
\end{equation*}}

In a simplistic setting where $\mathrm{Pa}_{\mathcal{C}}(X)=\{A_{1:K}\}$, $\pr^{\doint_j}(X\mid A_{1:K}) = \pr(X\mid A_{1:K})$ by invariance of causal mechanisms.
Consequently, $\pr^{\doint_j}(A_{1:K}=a_{1:K}\mid X)$ simplifies to:
{\small
\begin{equation*}
\begin{aligned}
\pr^{\doint_j}\!\left(A_{1:K}=a_{1:K}\mid X\right)
&\propto
\frac{\pr^{\doint_j}\!\left(A_{1:K}=a_{1:K}\right)}
     {\pr\!\left(A_{1:K}=a_{1:K}\right)}
\cdot
\pr\!\left(A_{1:K}=a_{1:K}\mid X\right).
\end{aligned}
\end{equation*}}
Here, the ratio term can be computed via causal and probabilistic inference in the causal \pc, and the remaining term by the attribute predictor. Hence, $\pr^{\doint_j}(Y\mid X)$ can be obtained from the two~modules.

Nevertheless, we typically do not know the causal relationships between the high-dimensional input $X$ and other variables, and $\pr^{\doint_j}(X\mid A_{1:K})$ does not coincide with $\pr(X\mid A_{1:K})$ in general.
This mismatch can arise, for example, if $X$ has additional (possibly unobserved) causal parents beyond $A_{1:K}$, or if some attributes in $A_{1:K}$ are descendants of $X$.
Therefore, modeling $\pr^{\doint_j}(X\mid A_{1:K})$ would require learning an interventional generative model for $X$, which is challenging, as many powerful generative models (\eg, diffusion models~\citep{ddpm} and GANs~\citep{gan}) do not allow access to the generative probabilities but only sampling, and we don't have access to $X$'s interventional samples.

\subsection{Intervention propagation: causal neural probabilistic circuits} \label{sec:cnpc}
In this section, we propose the Causal Neural Probabilistic Circuit (\cnpc), which consists of an \textit{attribute predictor} and a \textit{causal \pc}; see the top-left and bottom modules in Figure~\ref{fig:framework}. 
Similar to \npc, the attribute predictor is a multi-head neural network that predicts the distributions over the $K$ attributes from the input, and is trained with standard supervised learning. 
The causal \pc is obtained by compiling the causal structure over the attributes and the class label, and its parameters are given by the maximum-likelihood estimates. 

Compared with \npc, the main architectural difference is that \cnpc replaces the \pc with a causal \pc. However, this replacement alone is not sufficient to induce an interventional class distribution that accounts for causal dependencies. We still need to address the issue described in Section~\ref{sec:hardness}, namely modeling $\pr^{\doint_j}\!\left(A_{1:K}\mid X\right)$. Since the exact interventional conditional $\pr^{do_j}(A_{1:K}\mid X)$ is unavailable, we approximate it by combining two tractable distributions that capture complementary information: \textbf{(i)}~the clamped predictive distribution $\pr_{\theta}^{do_j}(A_{1:K}\mid X)$ from the neural predictor, which preserves evidence from the input but does not causally propagate the intervention; and \textbf{(ii)}~the interventional marginal $\pr_{w}^{do_j}(A_{1:K})$ from the causal \pc, which correctly propagates the intervention through the causal structure but is not conditioned on the input.
To combine the two, we adopt the reverse KL-barycenter approximation~\citep{genest1986combining,heskes1998selecting}, defined as the solution to the variational problem of minimizing a weighted sum of reverse KL divergences from each source:
{
\small
\begin{equation*}
    \tilde{\pr}_{\theta,w,\alpha}^{\doint_j}(\cdot \mid X)
:=
\arg\min_{q \in \Delta(\mathcal{A}_{1:K})}
(1-\alpha)\,
\KL\!\left(
q \,\middle\|\, \pr_{\theta}^{\doint_j}(\cdot \mid X)
\right)
+
\alpha\,
\KL\!\left(
q \,\middle\|\, \pr_{w}^{\doint_j}(\cdot)
\right),
\end{equation*}
}
where $\Delta(\mathcal{A}_{1:K})$ denotes the probability simplex over the joint attribute space, $\alpha \in [0,1]$ controls the relative influence of the two sources. 
The unique minimizer is the normalized geometric mean, also known as a product of experts (PoE), given as follows: 
{\small
\begin{equation}
\label{eq:cnpc_attr}
\begin{aligned}
\tilde{\pr}^{\doint_j}_{\theta,w,\alpha}( A_{1:K}=a_{1:K} \mid X )
&= \frac{\left( \pr^{\doint_j}_\theta( A_{1:K}=a_{1:K} \mid X ) \right)^{1-\alpha} \left( \pr^{\doint_j}_w( A_{1:K}=a_{1:K} ) \right)^{\alpha}}{Z_\alpha(X)},
\end{aligned}
\end{equation}}
where $Z_\alpha(X)$ is the normalizing constant.\footnote{The extension to interventions on multiple attributes is straightforward, as we only need to clamp the intervened attributes to their ground-truth distributions and compute the corresponding interventional marginal using the causal \pc.} 
With this PoE approximation, \cnpc predicts the interventional class distribution as:
{\small
\begin{equation}
\label{eq:cnpc_int}
\begin{aligned}
\tilde{\pr}^{\doint_j}_{\theta,w,\alpha}( Y \mid X )
&= \sum_{a_{1:K}} \pr_w(Y \mid A_{1:K}=a_{1:K}) \cdot \tilde{\pr}^{\doint_j}_{\theta,w,\alpha}( A_{1:K}=a_{1:K} \mid X ).
\end{aligned}
\end{equation}
}
Finally, \cnpc predicts the attribute labels as $\hat{a}_{1:K} = \arg\max_{a_{1:K} \in \mathcal{A}_{1:K}} \tilde{\mathbb{P}}_{\theta,w,\alpha}^{\mathrm{do}_j}(A_{1:K}=a_{1:K}\mid X)$ and the class label as $\hat{y} = \arg\max_{y\in\mathcal{Y}} \tilde{\mathbb{P}}_{\theta,w,\alpha}^{\mathrm{do}_j}(Y=y\mid X)$.
When there is no intervention, \cnpc~predicts the class distribution in the same way as \npc, \ie, via Eq.~\eqref{eq:npc}. 
\cnpc reduces to \npc when $\alpha=0$.
Overall, the PoE approximation allows \cnpc to preserve the evidence from the input through the neural predictor, while accounting for causal dependencies through the causal \pc, whose interventional marginal inherently captures the effects of intervening on one attribute on the others.

\section{Theoretical analysis} \label{sec:theory}

In this section, we present theoretical results to understand the behavior of \npc and \cnpc.
First, analogous to~\citet[Theorem~2]{npc}, \Cref{thm:benign_comp} in Appendix~\ref{app:theory} shows that, in the absence of interventions, the prediction errors of both \npc and \cnpc admit a compositional bound in terms of the errors of the attribute predictor and the \pc. As a result, improving either module improves the overall predictive distribution.
Corollaries~\ref{thm:int_comp_npc} and \ref{thm:int_comp_cnpc} further show that a similar compositional error decomposition still holds under interventions.

\begin{corollary} \label{thm:int_comp_npc}
The expected interventional error of \npc is upper-bounded by the sum of the expected interventional error of the attribute predictor and the expected prediction error of the \pc. Specifically,
{
\small
\begin{equation*}
\begin{aligned}
\bE_{X\sim\pr_*^{\doint_j}} \left[ \KL\left( \pr_*^{\doint_j}(Y\mid X) \|\ \pr^{\doint_j}_{\theta,w}(Y\mid X) \right) \right] 
&\le
\bE_{X\sim\pr_*^{\doint_j}} \left[ \KL\left( \pr_*^{\doint_j}(A_{1:K} \mid X) \|\ \pr^{\doint_j}_{\theta}(A_{1:K} \mid X) \right) \right] \\
&+ \bE_{A_{1:K}\sim\pr_*^{\doint_j}} \left[ \KL\left( \pr_*(Y\mid A_{1:K}) \|\ \pr_{w}(Y\mid A_{1:K}) \right) \right]. 
\end{aligned}
\end{equation*}
}
Equality holds if for each $(x,y)$, there exists a constant $c(x,y)$ s.t. $\frac{\pr_*^{\doint_j}(a_{1:K}\mid x)\pr_*(y\mid a_{1:K})}{\pr_{\theta}^{\doint_j}(a_{1:K}\mid x)\pr_{w}(y\mid a_{1:K})} = c(x,y)$ for all $a_{1:K}$.
\end{corollary}

\begin{corollary} \label{thm:int_comp_cnpc}
The expected interventional error of \cnpc is upper-bounded by a weighted sum of (i) the expected interventional error of the attribute predictor, (ii) the expected interventional error of the causal \pc over attributes, and (iii) the expected prediction error of the causal \pc. Specifically,
{
\small
\begin{equation*}
\resizebox{\linewidth}{!}{
$
\begin{aligned}
\bE_{X\sim\pr_*^{\doint_j}} \!\left[ \KL\!\left( \pr_*^{\doint_j}(Y\mid X)\ \|\ \tilde{\pr}^{\doint_j}_{\theta,w,\alpha}(Y\mid X) \right) \right] 
&\le
(1-\alpha)\ \bE_{X\sim\pr_*^{\doint_j}} \!\left[ \KL \!\left( \pr_*^{\doint_j}(A_{1:K}\mid X)\ \|\ \pr^{\doint_j}_{\theta}(A_{1:K}\mid X) \right) \right] \\
&+ \alpha\ \bE_{X\sim\pr_*^{\doint_j}} \!\left[ \KL \!\left( \pr_*^{\doint_j}(A_{1:K}\mid X)\ \|\ \pr^{\doint_j}_{w}(A_{1:K}) \right) \right] \\
&+ \bE_{A_{1:K}\sim\pr_*^{\doint_j}} \left[ \KL\left( \pr_*(Y\mid A_{1:K}) \|\ \pr_{w}(Y\mid A_{1:K}) \right) \right] .
\end{aligned}
$
}
\end{equation*}
}
Equality holds if for each $(x,y)$, $\pr^{\doint_j}_\theta(a_{1:K} \mid x)=\pr^{\doint_j}_w(a_{1:K})$ for all $a_{1:K}$, and there exists a constant $c(x,y)$ s.t. $\frac{\pr_*^{\doint_j}(a_{1:K}\mid x)\pr_*(y\mid a_{1:K})}{\pr_{\theta}^{\doint_j}(a_{1:K}\mid x)\pr_{w}(y\mid a_{1:K})} = c(x,y)$ for all $a_{1:K}$.
\end{corollary}

Both bounds inherit the compositional structure of the observational case.
Let $\mathbb{B}_{\npc}$ and $\mathbb{B}_{\cnpc}$ denote the interventional error bounds in Corollaries~\ref{thm:int_comp_npc} and~\ref{thm:int_comp_cnpc}, respectively. 
It follows directly that if
{\small
\begin{equation*}
\begin{aligned}
\KL \!\left( \pr_*^{\doint_j}(A_{1:K}\mid X)\ \|\ \pr^{\doint_j}_{w}(A_{1:K}) \right)
&\le
\KL \!\left( \pr_*^{\doint_j}(A_{1:K}\mid X)\ \|\ \pr^{\doint_j}_{\theta}(A_{1:K}\mid X) \right),
\end{aligned}
\end{equation*}}
then $\mathbb{B}_{\cnpc} \le \mathbb{B}_{\npc}$.
This result implies that, when the attribute predictor is unreliable, or the causal \pc's interventional marginal is close in KL to the ground-truth interventional attribute distribution, \cnpc might perform better than \npc.
In summary, our analysis establishes a compositional error bound for both models and identifies the conditions under which \cnpc improves upon \npc.



\section{Experiments} \label{sec:exp}
\subsection{Experimental settings}\label{sec:exp_setup}
\paragraph{Datasets.}
We evaluate on five datasets spanning diverse domains. All datasets are split 80:10:10 into training, validation, and test sets.
\textbf{(1) Asia}~\citep{asia} is a Bayesian network for medical diagnosis from bnlearn~\citep{bnlearn}. We sample 10{,}000 data points, taking the ``dysp'' node as the class variable and the remaining nodes as attributes. Following~\citet{c2bm}, we train an autoencoder to map attribute labels into an embedding, which serves as the input to models.
\textbf{(2) Sachs}~\citep{sachs} is a Bayesian network for biological signaling from bnlearn, processed identically to Asia. Since neither dataset contains images, we use them only in the benign setting.
\textbf{(3) MNISTAdd}~\citep{deepproblog}: Each instance pairs two digit images from MNIST~\citep{mnist}, with the digit values as attribute labels and their sum as the class label. We generate 50{,}000 instances by sampling labels from a predefined causal Bayesian network and then concatenating randomly sampled images.
\textbf{(4) cMNISTAdd}~\citep{cmnist} extends MNISTAdd with spurious correlations: digits 0--9 are paired with specific colors in training and validation, and this mapping is reversed at test time. We use it exclusively for OOD evaluation.
\textbf{(5) CelebA}~\citep{celeba} contains 202{,}599 celebrity face images with 40 imbalanced binary attributes. We use the five most balanced attributes and define the class label as a unique combination of attribute labels (yielding~32 classes), and manually annotate the underlying causal graph.
For MNISTAdd and CelebA, we construct two OOD variants by applying 180-degree rotations and PGD attacks~\citep{pgd_attack} to the standard test set.

\paragraph{Baselines.}
We select five representative CBMs as our baselines: 
\textbf{(1)~Vanilla CBM}~\citep{cbm} consists of a standard neural concept predictor followed by a linear label predictor. 
\textbf{(2)~CEM}~\citep{cem} utilizes high-dimensional concept embeddings rather than scalar concept predictions for its label predictor to enhance downstream task performance. 
\textbf{(3)~SCBM}~\citep{scbm} models the joint distribution of all concepts using a multivariate normal distribution and treats interventions as conditional distributions. 
\textbf{(4)~C$^2$BM}~\citep{c2bm} integrates a causal graph into the concept predictor, allowing interventions to propagate through the graph structure. 
\textbf{(5)~\npc}~\citep{npc} employs a \pc as the label predictor to exactly encode the probabilistic dependencies among attributes and the class label.

\paragraph{Evaluation metrics.}
We adopt two standard evaluation metrics from the CBM literature: \textit{task accuracy}, the model's accuracy in predicting the class label, and \textit{(mean) attribute accuracy}, the model's accuracy in predicting each attribute label, averaged across attributes.

\paragraph{Implementation.}
For a fair evaluation, \cnpc and all baseline models employ Multi-Layer Perceptrons (MLPs) as their neural predictors. The attribute intervention order is determined by node depth in the causal graph, \ie, the longest directed path from the root nodes. 
For \cnpc, the PoE weight $\alpha$ is selected on the in-distribution validation set in the benign setting; in the OOD setting, since OOD validation data is by assumption unavailable, we set $\alpha = 0.9$ across all datasets and all intervened attributes, placing more weight on evidence from the causal \pc.
An ablation study analyzing the effect of varying $\alpha$ is presented in Section~\ref{sec:ablation}.
More implementation details are provided in Appendix~\ref{app:exp_setup}.

\subsection{Quantitative results}
We evaluate models in the benign setting (Section~\ref{sec:benign}) and three distinct OOD settings (Section~\ref{sec:ood}): unseen data transformations, adversarial perturbations, and spurious-correlation shifts.

\subsubsection{Performance in the benign setting} \label{sec:benign}

\begin{figure*}[tb]
    \centering
    \setlength{\abovecaptionskip}{4pt}
    \setlength{\belowcaptionskip}{-8pt}
    \includegraphics[width=1\linewidth]{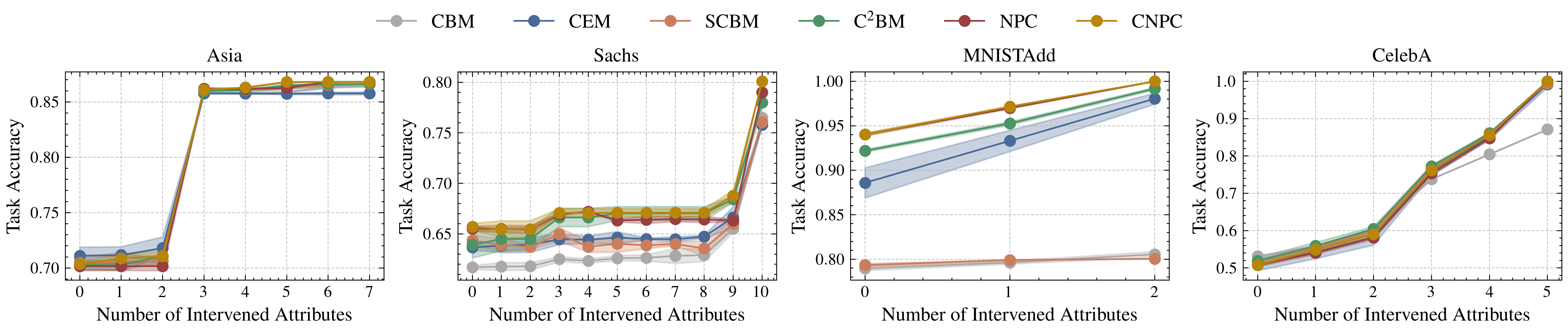}
    \caption{\small Task accuracy of all models in the benign setting on the Asia, Sachs, MNISTAdd, and CelebA datasets under varying numbers of attribute interventions. All results are averaged across three random seeds.}
    \label{fig:benign_task}
\end{figure*}

\begin{figure*}[tb]
    \centering
    \setlength{\abovecaptionskip}{4pt}
    \setlength{\belowcaptionskip}{-10pt}
    \includegraphics[width=1\linewidth]{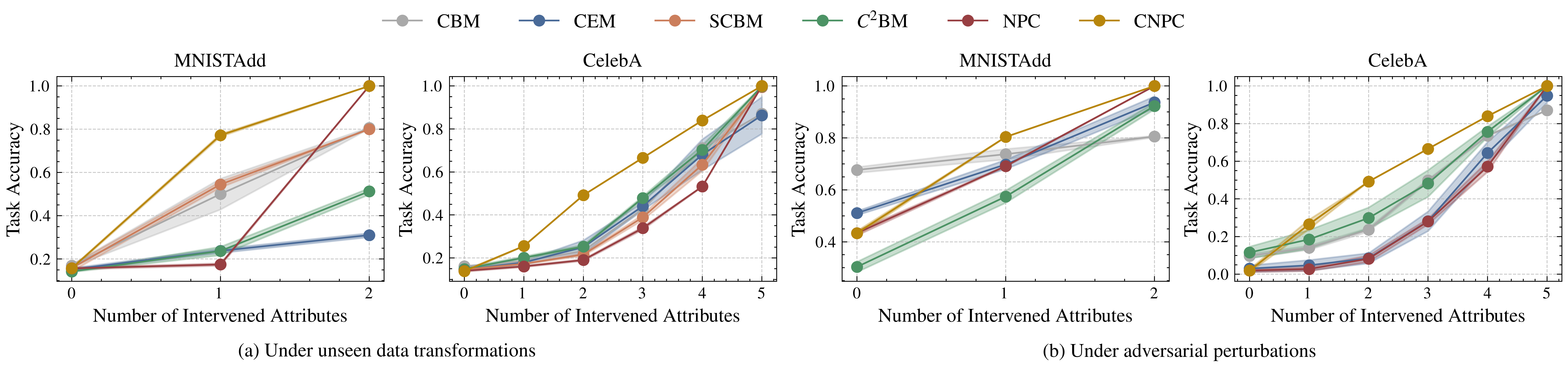}
    \caption{\small Task accuracy of all models in OOD settings on the MNISTAdd and CelebA datasets under varying numbers of attribute interventions. All results are averaged across three random seeds.}
    \label{fig:ood_task}
\end{figure*}

Figure~\ref{fig:benign_task} shows the task accuracy of all models in the benign setting on the Asia, Sachs, MNISTAdd, and CelebA datasets under varying numbers of attribute interventions. Results for attribute accuracy are shown in Figure~\ref{fig:benign_attr} in Appendix~\ref{app:benign}. 
Across all four datasets, every model shows an upward trend in task accuracy as the number of intervened attributes increases. This consistency demonstrates that interventions can reliably correct model predictions in concept bottleneck models, highlighting the value of the human-model collaboration these models enable. 

On MNISTAdd, \cnpc attains the highest task accuracy at every intervention count, while on the remaining datasets, all models perform comparably, with \cnpc marginally leading on Asia and Sachs and C$^2$BM on CelebA. Together, these results suggest that the models which incorporate the underlying causal structure, whether by compiling into a \pc or through neural approximations of structural equations, can improve intervention efficiency. Hence, they underscore the importance of modeling causal dependencies in concept bottleneck models.

\begin{figure*}[t]
    \centering
    \setlength{\abovecaptionskip}{2pt}
    \setlength{\belowcaptionskip}{-12pt}
    \includegraphics[width=0.85\linewidth]{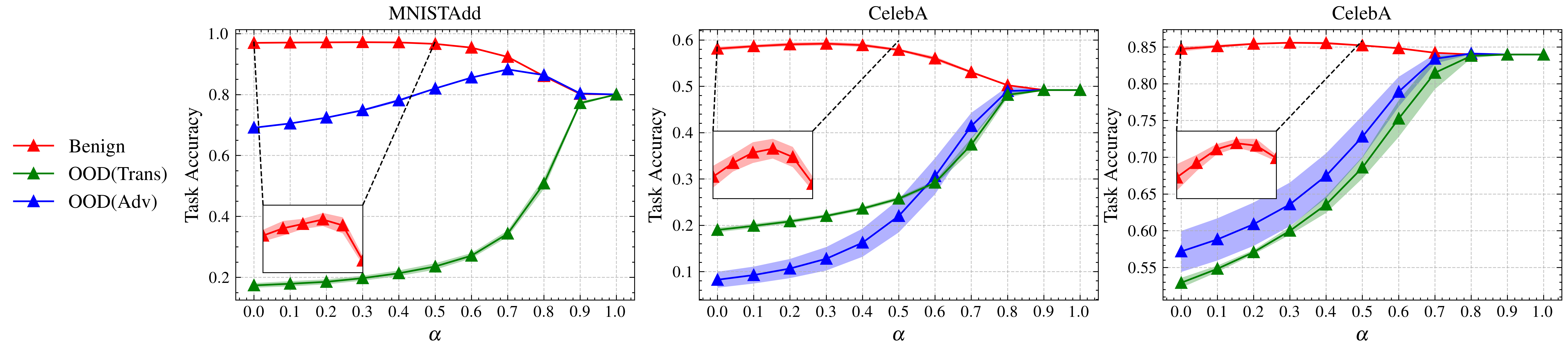}
    \caption{\small Task accuracy of \cnpc in both benign and OOD settings with $\alpha$ varying from 0.0 to 1.0 in increments of 0.1. \textbf{Left:} Performance on MNISTAdd with one intervened attribute. \textbf{Middle/Right:} Performance on CelebA with two/four intervened attributes. All results are averaged across three random seeds.
    }
    \label{fig:alpha}
\end{figure*}

\subsubsection{Performance in the OOD settings} \label{sec:ood}
\paragraph{Unseen data transformations.}
Figure~\ref{fig:ood_task} (a) reports the task accuracy on MNISTAdd and CelebA under unseen data transformations across varying numbers of attribute interventions; the corresponding attribute accuracy is shown in Figure~\ref{fig:ood_attr}~(a) in Appendix~\ref{app:ood}.
These transformations severely degrade attribute accuracy across all models: before any intervention, accuracy drops from over 92\% to below 20\% on MNISTAdd and from over 85\% to roughly 50\% on CelebA. Interventions therefore become crucial for recovering task performance. 
On both datasets, \cnpc consistently outperforms all baselines by a substantial margin across all intervention counts. For example, it surpasses the second-best model by 23\% on MNISTAdd with one intervened attribute and by 24\% on CelebA with two. These results highlight \cnpc's strong intervention efficiency under unseen transformations.

\paragraph{Adversarial perturbations.}
Figure~\ref{fig:ood_task} (b) reports task accuracy on MNISTAdd and CelebA across varying numbers of attribute interventions under adversarial perturbations that corrupt the neural predictor's outputs for all attributes; the attribute accuracy is shown in Figure~\ref{fig:ood_attr} (b) in Appendix~\ref{app:ood}. 
We exclude SCBM from this experiment because it relies on non-differentiable discrete sampling for concept predictions and is therefore incompatible with the gradient-based PGD attack. 
As with unseen data transformations, adversarial perturbations substantially degrade attribute accuracy across all models. 
We find that, although \cnpc starts from a relatively low task accuracy before intervention, it capitalizes on interventions effectively and ultimately attains the highest task accuracy at every intervention count, which further corroborates its strong efficiency in utilizing interventions.

\begin{wrapfigure}{r}{0.50\textwidth}
    \vspace{-8pt}
    \centering
    \setlength{\belowcaptionskip}{-10pt}
    \setlength{\abovecaptionskip}{2pt}
    \begin{subfigure}[b]{0.23\textwidth}
        \centering
        \includegraphics[width=\textwidth]{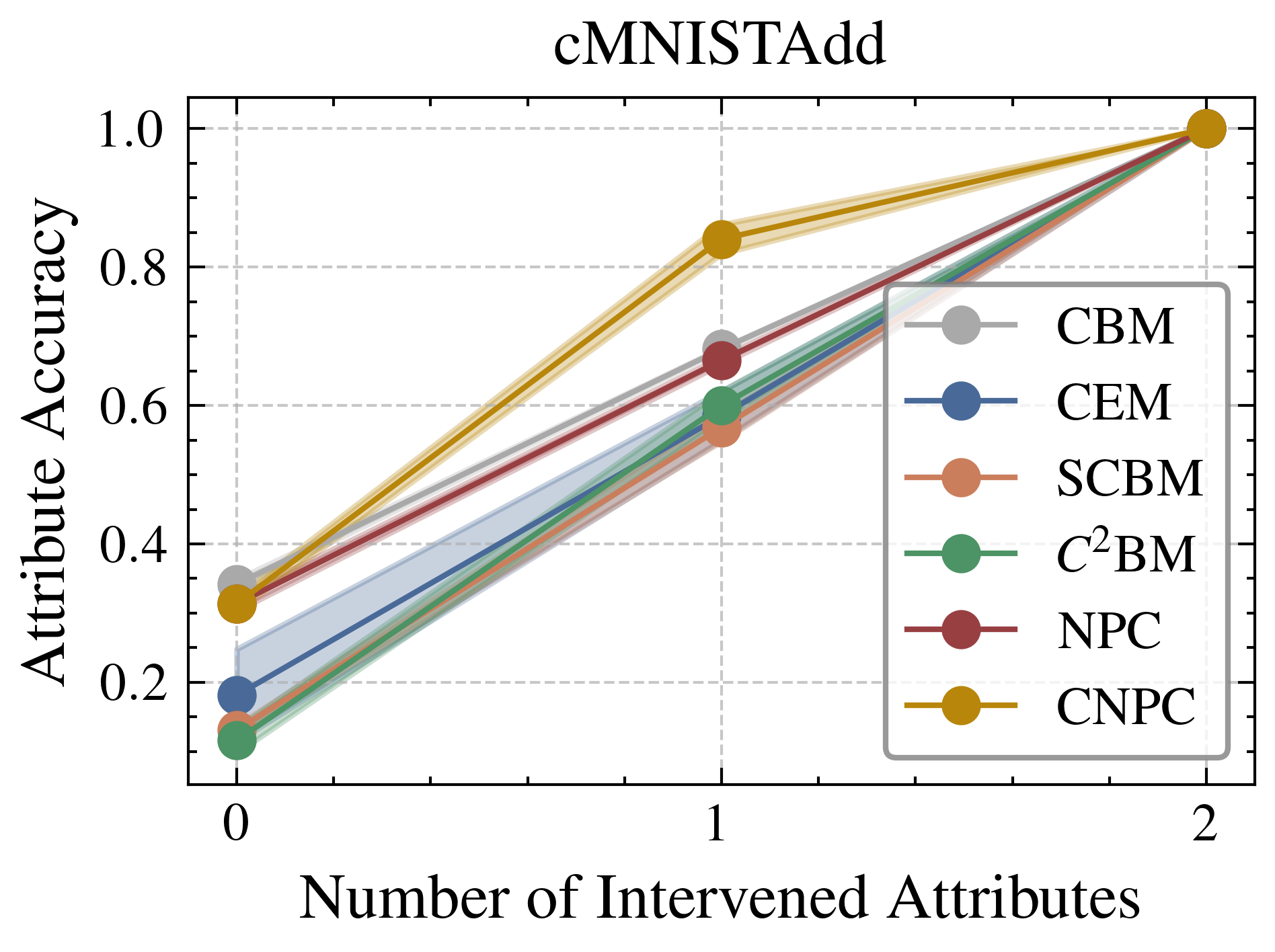}
    \end{subfigure}
    \begin{subfigure}[b]{0.23\textwidth}
        \centering
        \includegraphics[width=\textwidth]{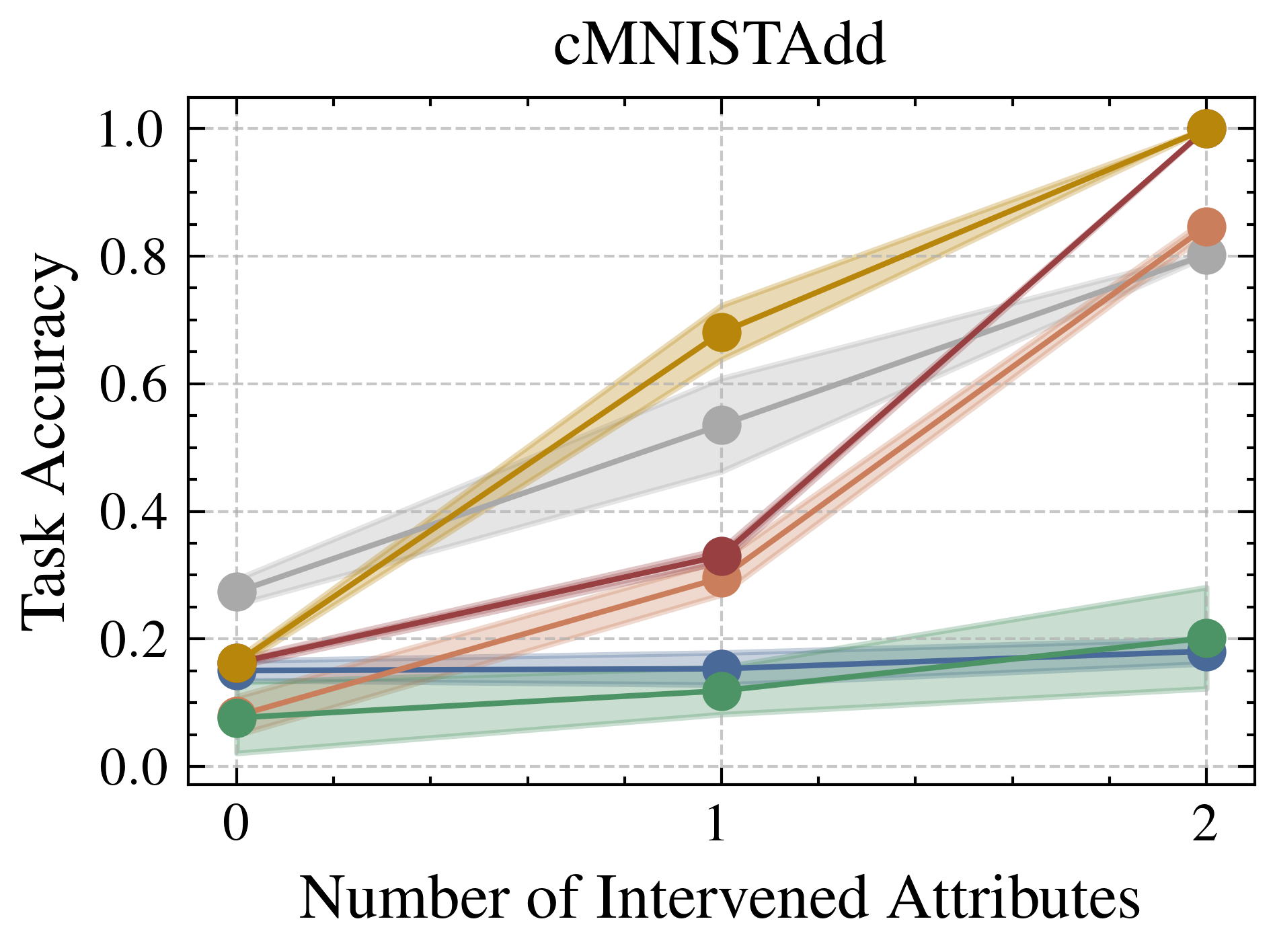}
    \end{subfigure}
    \caption{\small Attribute accuracy (\textbf{left}) and task accuracy (\textbf{right}) of all models on cMNISTAdd. 
    }
    \label{fig:cmnist_acc}
\end{wrapfigure}

\paragraph{Spurious-correlation shifts.}
Figure~\ref{fig:cmnist_acc} reports performance on cMNISTAdd, where digit values are spuriously correlated with colors. As with the previous shifts, spurious correlations severely degrade the attribute accuracy before intervention across all models, dropping it from over 92\% to below 40\%. Nonetheless, \cnpc effectively uses interventions, achieving the highest attribute and task accuracy after intervention.


\subsection{Ablation studies} \label{sec:ablation}
\paragraph{Effect of $\alpha$.}


To systematically examine the effect of $\alpha$ on \cnpc's performance, we report in Figure~\ref{fig:alpha} \cnpc's task accuracy across varying $\alpha$ in three scenarios: MNISTAdd with one intervened attribute, and CelebA with two and four intervened attributes.
In the benign setting, task accuracy first rises and then falls as $\alpha$ grows, peaking at $\alpha=0.3$ in all three scenarios. Notably, the neural predictor achieves near-perfect attribute accuracy in these scenarios. Together, these findings suggest that \textit{even when the neural predictor is highly reliable, relying on it exclusively is suboptimal; judiciously incorporating information from the causal \pc can improve interventional performance.}
Under unseen data transformations, task accuracy increases monotonically with $\alpha$. Under adversarial attacks, it follows a similar inverted-U pattern as in the benign setting albeit peaking much later: at $\alpha=0.7$ on MNISTAdd, $\alpha=0.9$ on CelebA with two intervened attributes, and $\alpha=0.8$ on CelebA with four. These results indicate that \textit{even when the neural predictor is compromised, retaining a modest contribution from it still benefits interventional performance.}

Overall, preserving the neural predictor's evidence for the input and accounting for causal dependencies encoded in the \pc are both important, and an appropriate $\alpha$ is needed to balance these two sources of evidence. This naturally raises a follow-up question: \textit{is there a mechanism to adaptively select the optimal $\alpha$ for a given sample?} We discuss this possibility in Appendix~\ref{sec:discuss}.


\section{Conclusions}
To exploit the causal relationships among attributes, we propose \cnpc, a model that combines a neural predictor with a \pc compiled from a causal graph and estimates the interventional distribution via a PoE.
Both theoretical and empirical evidence show that \cnpc effectively improves the intervention efficiency over \npc, especially under OOD settings.
More broadly, \cnpc highlights a practical way to incorporate causal information into predictive models and to approximate causal inference within these models. Limitations and potential solutions are discussed in Appendix~\ref{sec:discuss}.

\bibliographystyle{unsrtnat}
\bibliography{neurips_2026}


\newpage
\appendix

\section{Discussion} \label{sec:discuss}
In this section, we discuss the limitations of the proposed method and outline potential solutions, which are categorized into the following three perspectives.

\paragraph{Reducing prerequisites.}
Our method assumes access to data annotated with both attribute and class labels, along with the underlying causal structure over these variables. We adopt this setting because our primary goal is to study how causal relationships can be effectively incorporated into and leveraged by concept bottleneck models. In practice, obtaining such prerequisites may be challenging, but several lines of work address each requirement. LLMs and VLMs have recently been used for automated attribute annotation in concept bottleneck models \citep{label_free_cbm, posthoc_cbm, vlg_cbm}, and a range of well-established causal discovery algorithms can recover plausible causal structures from observational data \citep{causal_discovery_pc, causal_discovery_hillclimb, causal_discovery_ges}. These tools can be incorporated into our pipeline as preprocessing steps when full prerequisites are unavailable, and we leave a systematic study of \cnpc under such relaxations to future work.

\paragraph{Improving \cnpc.}
We outline two directions for further improving \cnpc.

\textbf{(i) Adaptive selection of $\alpha$.}
In our main experiments, $\alpha$ is selected on the validation set in the benign setting and held fixed in the OOD settings. For each setting, all samples in the test set share the same value of $\alpha$ at test time. This setup allows us to isolate the contribution of the PoE formulation itself, without entangling it with a separate adaptation mechanism. A natural question, however, is whether $\alpha$ can be determined adaptively for each input. One promising approach is to tie $\alpha$ to a score from existing OOD detection algorithms~\citep{hendrycks2017baseline, liang2018odin}, so that inputs flagged as likely OOD receive a larger $\alpha$, placing more weight on the causal \pc and less on the potentially unreliable neural predictor. We leave a thorough study of such schemes, including their interaction with the calibration of OOD scores, to future work.

\textbf{(ii) Optimal ordering of attribute interventions.}
In our main experiments, attributes are intervened on in order of their depth in the causal graph, under the heuristic that higher-level attributes influence more downstream attributes and thus yield larger improvements on the target. This heuristic is motivated by the graph topology alone and is not claimed to be optimal. Identifying the truly optimal intervention set under a budget  (\ie, a fixed number of intervened attributes) likely depends on the specific causal mechanisms in addition to the topology. For instance, intervening on a causal parent of the target whose causal mechanism is near-deterministic may produce a substantial gain. A formal characterization of the optimal ordering, and an algorithm to compute it, are left to future work.

\paragraph{Solving more challenging problems.}
This work primarily addresses how intervening on a specific attribute (by setting it to its ground-truth label) influences the class distribution for a given instance $x$, \ie, computing $\Pr^{\doint(A_j:=a_j^*)}(Y\mid X=x)$, where $a_j^*$ is the ground-truth label of $x$'s $j$-th attribute. However, researchers may also be interested in counterfactual questions, such as \textit{``What would the outcome have been had the attribute been set to a different value?''}. This requires computing $\Pr^{\doint(A_j:=a_j^\prime)}(Y\mid X=x, A_j=a_j^*)$, where $a_j^\prime \neq a_j^*$. As noted in Section~\ref{sec:cpc}, the causal \pc supports exact, tractable counterfactual inference. Therefore, future work could leverage the outputs of the causal \pc to develop approximations of ground-truth counterfactual class distributions.

\section{Broader Impacts} \label{app:broad}
This work studies a new concept bottleneck model and inference procedure for improving intervention efficiency. Concept bottleneck models are designed to improve the interpretability of black-box neural networks by exposing intermediate concept predictions, and interventions further enable correcting mispredicted concept values before the final prediction is made. Both mechanisms aim to enhance model reliability, and we therefore do not anticipate direct negative societal impacts from the method itself.

\section{Preliminaries} \label{app:prelim}

\subsection{Causal probabilistic circuits} \label{app:cpc}
In this section, we elaborate on how VE schedules the factor operations (multiplication and sum-out) within a network polynomial, which in turn determines the structure of the compiled arithmetic circuit.

We continue the example from Section~\ref{sec:cpc}. That is, we consider three variables with the graph structure $V_2 \leftarrow V_1 \rightarrow V_3$, and apply VE with the variable elimination order $[V_2,\ V_3,\ V_1]$. 
Table~\ref{tab:ve} summarizes the VE schedule.
At each step, VE collects all current factors that involve the variable to be eliminated, multiplies them, and then sums out that variable.
This produces a new factor over the remaining variables, which is carried forward to subsequent steps.
Specifically, the new factors produced at different steps are: $\tau_1(V_1) = \sum_{V_2} \phi_{V_2}(V_1, V_2)\ \lambda_{V_2}(V_2)$, $\tau_2(V_1) = \sum_{V_3} \phi_{V_3}(V_1, V_3) \ \lambda_{V_3}(V_3)$, and the network polynomial $f = \sum_{V_1} \phi_{V_1}(V_1)\ \tau_1(V_1)\ \tau_2(V_1)\ \lambda_{V_1}(V_1)$.

Finding the optimal variable elimination order is known to be NP-hard. In our experiments, the variable elimination order is determined using the MinFill heuristic, which greedily selects the variable that adds the fewest new edges to the graph when eliminated.

Subsequently, the compiled arithmetic circuit is obtained by mapping each sum-out operation to a sum node and each multiplication operation to a product node, while leaf nodes correspond to symbolic parameters (the $\phi$ entries) and indicators (the $\lambda$ entries).

\begin{table}[htbp]
    \centering
    \caption{A run of VE for the network polynomial with the elimination order $[V_2,\ V_3,\ V_1]$.}
    \begin{tabular}{c|c|c|c|c}
        \toprule
         Step &  Variable   &  Factors                                                            &  Variables   & New            \\
              &  eliminated &  used                                                               &  involved    & factor         \\
        \midrule
         1    &  $V_2$      &  $\phi_{V_2}(V_1, V_2),\ \lambda_{V_2}(V_2)$                        &  $V_1,\ V_2$ & $\tau_1(V_1)$  \\
         2    &  $V_3$      &  $\phi_{V_3}(V_1, V_3),\ \lambda_{V_3}(V_3)$                        &  $V_1,\ V_3$ & $\tau_2(V_1)$  \\
         3    &  $V_1$      &  $\phi_{V_1}(V_1),\ \tau_1(V_1),\ \tau_2(V_1),\ \lambda_{V_1}(V_1)$ &  $V_1$       & $f$\\
         \bottomrule
    \end{tabular}
    \label{tab:ve}
\end{table}

\section{Derivation of the interventional class distribution} \label{app:derivation}
Let $a_{1:K}|_{\mathrm{Pa}_{\mathcal{C}}(Y)}$ denote the instantiation of $\mathrm{Pa}_{\mathcal{C}}(Y)$ induced by $a_{1:K}$.

Following Assumption~\ref{assump:sufficient}, Assumption~\ref{assump:structure} together with the local Markov property, and the invariance of causal mechanisms, Eq.~\eqref{equ:trueint} can be derived as follows:

\begin{align*}
    \pr^{\doint_j}(Y\mid X) 
    &= \sum_{a_{1:K}} \pr^{\doint_j}(Y\mid A_{1:K}=a_{1:K}, X) \cdot \pr^{\doint_j}(A_{1:K}=a_{1:K}\mid X),\\
    &= \sum_{a_{1:K}} \pr^{\doint_j}(Y\mid A_{1:K}=a_{1:K}) \cdot \pr^{\doint_j}(A_{1:K}=a_{1:K}\mid X),\\
    &= \sum_{a_{1:K}} \pr^{\doint_j}(Y\mid \mathrm{Pa}_{\mathcal{C}}(Y)=a_{1:K}|_{\mathrm{Pa}_{\mathcal{C}}(Y)}) \cdot \pr^{\doint_j}(A_{1:K}=a_{1:K}\mid X),\\
    &= \sum_{a_{1:K}} \pr(Y\mid \mathrm{Pa}_{\mathcal{C}}(Y)=a_{1:K}|_{\mathrm{Pa}_{\mathcal{C}}(Y)}) \cdot \pr^{\doint_j}(A_{1:K}=a_{1:K}\mid X),\\
    &= \sum_{a_{1:K}} \pr(Y\mid A_{1:K}=a_{1:K}) \cdot \pr^{\doint_j}(A_{1:K}=a_{1:K}\mid X).
\end{align*}

The first equality follows from the law of total probability. The second equality uses Assumption~\ref{assump:sufficient}, which states that $Y \perp X \mid A_{1:K}$. The third equality follows from Assumption~\ref{assump:structure} and the local Markov property: since $\mathrm{Pa}_{\mathcal{C}}(Y) \subseteq \{A_{1:K}\} \subseteq \mathrm{ND}_{\mathcal{C}}(Y)$ and $Y$ is conditionally independent of the non-descendants given $\mathrm{Pa}_{\mathcal{C}}(Y)$, conditioning on all attributes can be reduced to conditioning on the causal parents of $Y$. The fourth equality uses the invariance of causal mechanisms: the intervention $\doint_j$ acts on $A_j$ and does not change the causal mechanism of $Y$ (\ie, the distribution of $Y$ given its causal parents). The last equality again follows from Assumption~\ref{assump:structure} and the local Markov property.

\section{Experiments} \label{app:exp}
\subsection{Experimental settings} \label{app:exp_setup}
\subsubsection{Datasets} \label{app:datasets}
Here, we provide further details regarding the data processing procedure for each dataset.

\paragraph{Asia.} Following~\citet{c2bm}, we train an autoencoder to map the attribute labels into an embedding, which serves as the input to CBMs. Specifically, this autoencoder comprises two encoder layers and two decoder layers, with a latent dimension of 32. 
It takes a vector of all attribute labels for a given instance as input and is trained using Mean Squared Error (MSE) loss to reconstruct this vector. 
After training, the final instance embedding is constructed by combining 50\% of the encoder's output with 50\% noise sampled from a standard normal distribution.
This noise is injected to ensure the resulting embeddings are non-trivial representations of attributes.
Finally, the embeddings are standardized.

\paragraph{Sachs.} The Sachs dataset is processed using the identical procedure described for Asia.

\paragraph{MNISTAdd.}
The data-generating process for an instance in the MNISTAdd dataset proceeds in three steps. First, the value for attribute $A_1$ (the first digit) is drawn from a uniform distribution over 0–9. Second, the value for attribute $A_2$ (the second digit) is drawn from a highly skewed conditional distribution: with 80\% probability, it is assigned the value $(A_1 + 1)\%10$, while the remaining 20\% probability is distributed equally among the other possible digits. Finally, the class label is defined as their sum, and the instance image is formed by concatenating randomly sampled MNIST images corresponding to the sampled digit values. The causal graph depicting the relationships between the attributes and the class is illustrated in Figure~\ref{fig:causal_graph_mnistadd}.

Additionally, we generate two distinct test set variants to study model performance under OOD settings: (i) We apply a 180-degree rotation to each instance. (ii) We employ the PGD attack ($\epsilon=8/255$, $\alpha=2/255$, $50$ steps) to generate an adversarial perturbation for each instance.

\begin{figure}[tb]
    \centering
    \begin{subfigure}[b]{0.15\textwidth}
        \centering
        \includegraphics[width=\textwidth]{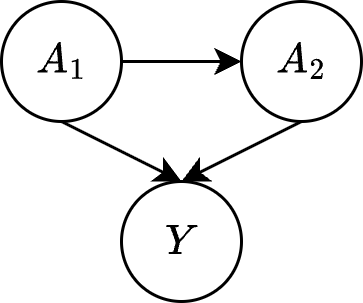}
        \caption{MNISTAdd}
        \label{fig:causal_graph_mnistadd}
    \end{subfigure}
    \begin{subfigure}[b]{0.40\textwidth}
        \centering
        \includegraphics[width=\textwidth]{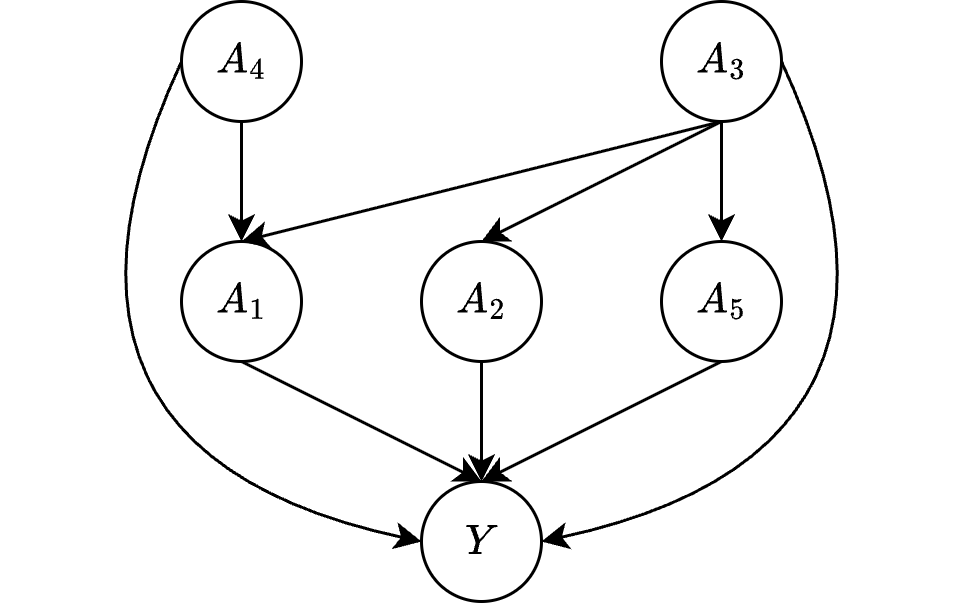}
        \caption{CelebA}
        \label{fig:causal_graph_celeba}
    \end{subfigure}
    
    \caption{Causal graphs over attributes $A_{1:K}$ and the class label $Y$. \textbf{Left:} The causal graph for MNISTAdd, with attributes $A_1$: the first digit and $A_2$: the second digit. \textbf{Right:} The causal graph for CelebA, with attributes $A_1$: Attractive, $A_2$: Mouth\_Slightly\_Open, $A_3$: Smiling, $A_4$: Wearing\_Lipstick, and $A_5$: High\_Cheekbones.}
    \label{fig:causal_graph}
\end{figure}

\paragraph{cMNISTAdd.}
The data-generating process for the cMNISTAdd dataset is similar to that for MNISTAdd, except that it introduces spurious correlations using colored digits. In the training and validation sets, the digits 0–9 are spuriously correlated with a specific list of colors; in the test set, this color-to-digit mapping is reversed.
Examples of instances from the training and test sets are illustrated in Figure~\ref{fig:cmnist}.

\begin{figure}[b]
    \centering
    \includegraphics[width=0.5\linewidth]{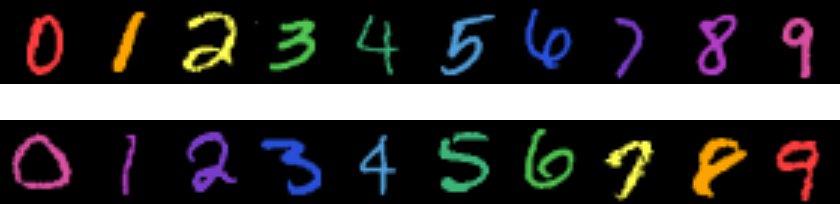}
    \caption{Instances from the cMNISTAdd dataset. \textbf{Top:} Examples from the training set. \textbf{Bottom:} Examples from the test set.}
    \label{fig:cmnist}
\end{figure}

\paragraph{CelebA.}
In this paper, we focus on the five most balanced attributes in CelebA: \textit{Attractive}, \textit{Mouth\_Slightly\_Open}, \textit{Smiling}, \textit{Wearing\_Lipstick}, and \textit{High\_Cheekbones}. We define the class label as a unique combination of these binary attributes, resulting in 32 distinct classes. Furthermore, we manually annotate the underlying causal graph over these attributes and the class label, as illustrated in Figure~\ref{fig:causal_graph_celeba}.

Additionally, we generate two distinct test set variants to study model performance under OOD settings: (i) We apply a 180-degree rotation to each instance. (ii) We employ the PGD attack ($\epsilon=2/255$, $\alpha=2/255$, $50$ steps) to generate an adversarial perturbation for each instance.

\paragraph{Licenses and terms of use.}
The Asia and Sachs datasets are generated by sampling from Bayesian networks provided by \texttt{bnlearn}, an R package distributed under GPL-2 or GPL-3. 
The MNISTAdd and cMNISTAdd datasets are constructed from MNIST images, which are distributed under the Creative Commons Attribution-ShareAlike 3.0 license. 
For CelebA, we use the images and attribute annotations only for non-commercial research purposes, in accordance with the official CelebA terms of use.

\subsubsection{Implementation details}\label{app:implementation_details}
To ensure the reliability and reproducibility of our results, we conduct all experiments across three independent trials using the random seeds 42, 52, and 62.
All experiments are conducted on NVIDIA RTX A6000 GPUs, each with 48 GB of memory.
Unless otherwise stated, for datasets containing image inputs, the images are resized to 224$\times$224$\times$3.

\paragraph{\cnpc.}
The attribute predictor of \cnpc is parameterized as a multi-task MLP consisting of a single shared hidden layer with a ReLU activation and independent linear classification heads for each attribute. The hidden dimension is set to 128.
We train the predictor for 100 epochs with a batch size of 256, using the average normalized cross-entropy loss across attributes from \citet{npc} as the training objective. We use the SGD optimizer with a learning rate of 0.01, a momentum of 0.9, and a weight decay of 4e-5.

\paragraph{\npc.}
\npc uses the attribute predictor architecture and training configuration described above. Its label predictor is a \pc whose structure is learned using LearnSPN~\citep{learnspn} and whose parameters are optimized using CCCP~\citep{cccp}. For a fair and direct comparison with \cnpc, we omit the final joint training stage of \npc proposed by \citet{npc}.

\paragraph{Vanilla CBM.}
In the vanilla CBM, the concept predictor is implemented as an MLP with a single hidden layer of dimension 128 and a ReLU activation, and the label predictor as a linear layer. The two predictors are trained independently, as \citet{cbm} show that this training paradigm achieves the highest intervention efficiency among the training paradigms they consider. The concept predictor is trained with binary cross-entropy over attribute labels, and the label predictor is trained separately with cross-entropy over class labels using ground-truth concepts. Both stages are trained for 100 epochs with a batch size of 256, using SGD with a learning rate of 0.01, momentum of 0.9, and weight decay of 4e-5.

\paragraph{CEM.}
In CEM, the concept predictor consists of a linear layer with a ReLU activation, followed by a set of concept blocks that produce concept probabilities and concept embeddings. The resulting high-dimensional concept embeddings are concatenated and passed through a linear layer to produce the final class predictions. 
We train CEM end-to-end for 100 epochs with a batch size of 256.
The training loss is the sum of the concept prediction loss and the class prediction loss, with the concept loss weight set to 1. 
We optimize the model using SGD with a learning rate of 0.01, a momentum of 0.9, and a weight decay of 4e-5.
We also adopt RandInt, a regularization strategy proposed by~\citet{cem}, which randomly applies independent concept interventions during training with probability 0.25 to improve intervention efficiency at inference time.

\paragraph{C$^2$BM and SCBM.}
To faithfully reproduce the results reported by \citet{c2bm}, we implement C$^2$BM and SCBM using the open-source code\footnote{\url{https://github.com/gdefe/causally-reliable-cbm}}.
For C$^2$BM, the input is encoded by a one-hidden-layer MLP, and each attribute and the class variable are further modeled by a block that produces embeddings and categorical probabilities.
The model is trained end-to-end with the sum of the attribute prediction loss and the class prediction loss, with the attribute loss weight set to 0.8.
During training, independent attribute interventions are randomly applied with probability 0.8.

For SCBM, the input is encoded by a one-hidden-layer MLP, after which the model predicts the mean of a multivariate distribution over concept logits and a learned global covariance matrix.
The model is trained end-to-end with a loss consisting of a class prediction loss, a Monte Carlo concept prediction loss based on binary cross-entropy, and an $\ell_1$ penalty on the off-diagonal entries of the precision matrix, with both the concept loss weight and the precision regularization weight set to 1.
Detailed training and evaluation configurations for \cnpc and all baseline models are provided in the supplementary code.

\subsection{Main results} \label{app:main}
\subsubsection{Performance in the benign setting} \label{app:benign}
Figure~\ref{fig:benign_attr} reports the mean attribute accuracy of \cnpc and baseline models in the benign setting on the Asia, Sachs, MNISTAdd, and CelebA datasets under varying numbers of attribute interventions. 
Since C$^2$BM, by construction, only predicts attributes that lie on the ancestral path from the class variable, attributes outside the class variable's ancestral subgraph are not predicted and have undefined accuracy. There exist such attributes on the Asia and Sachs datasets, \eg, ``xray'', ``Jnk'', ``Plcg'', and we therefore do not report C$^2$BM's mean attribute accuracy on these datasets.

We observe that all models perform comparably, with \cnpc marginally leading on Asia, Sachs, and MNISTAdd, and C$^2$BM marginally leading on CelebA. Together, these results suggest that models which incorporate the underlying causal structure, whether by compiling into a \pc or through neural approximations of structural equations, effectively propagate interventions to related attributes, yielding higher attribute accuracy under intervention.

\begin{figure*}[htbp]
    \centering
    \includegraphics[width=1\linewidth]{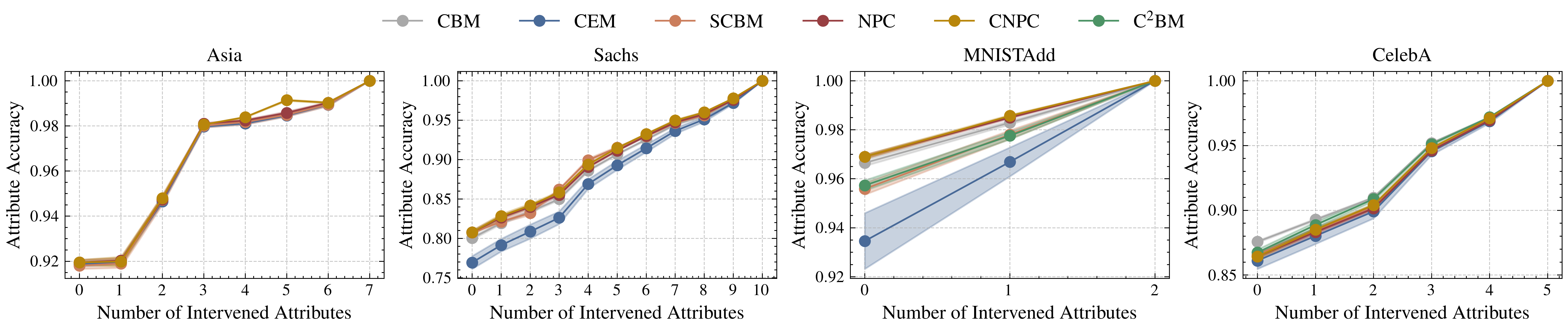}
    \caption{Attribute accuracy of \cnpc and baseline models in the benign setting on the Asia, Sachs, MNISTAdd, and CelebA datasets under varying numbers of attribute interventions. All results are averaged across three random seeds.}
    \label{fig:benign_attr}
\end{figure*}

\begin{figure*}[htbp]
    \centering
    \includegraphics[width=1\linewidth]{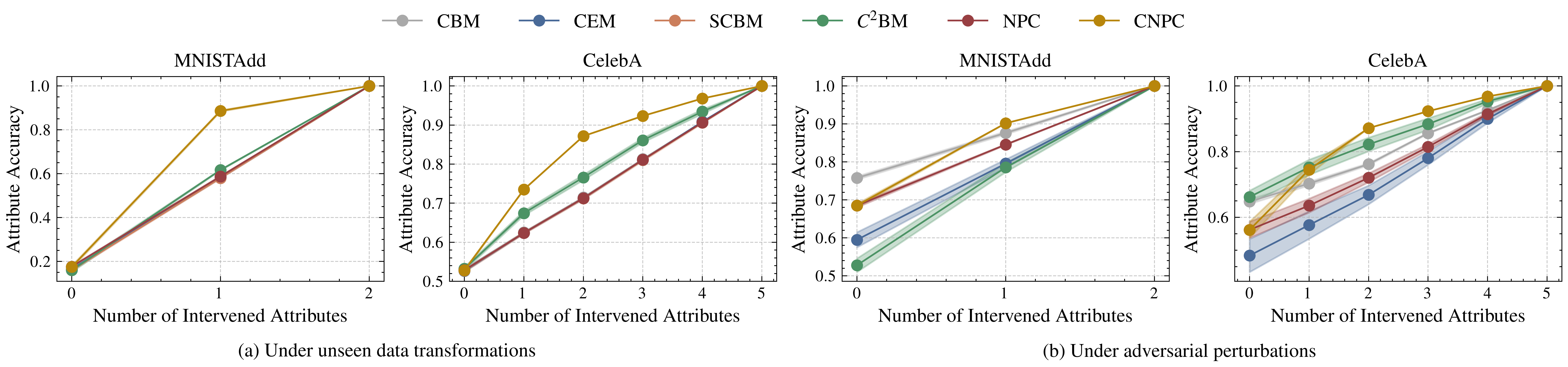}
    \caption{Attribute accuracy of \cnpc and baseline models in OOD settings on the MNISTAdd and CelebA datasets under varying numbers of attribute interventions. All results are averaged across three random seeds.}
    \label{fig:ood_attr}
\end{figure*}

\subsubsection{Performance in the OOD settings} \label{app:ood}
Figure~\ref{fig:ood_attr} reports the mean attribute accuracy of \cnpc and baseline models in the OOD settings on the MNISTAdd and CelebA datasets under varying numbers of attribute interventions. 
Before intervention, we find that the OOD shifts degrade attribute accuracy across all models. For example, the unseen data transformation drops attribute accuracy from over 92\% to below 20\% on MNISTAdd and from over 85\% to roughly 50\% on CelebA. Given this severe degradation in the neural predictor's reliability, interventions become essential for recovering downstream task performance. 

After intervention, we observe that \cnpc consistently attains the highest attribute accuracy across all OOD shifts, datasets, and intervention counts, demonstrating its effectiveness at propagating interventions through the causal structure to related attributes.

\section{Theoretical analysis} \label{app:theory}
In this section, we provide more theoretical results and elaborate on the proofs omitted in Section~\ref{sec:theory}.

\begin{theorem} \label{thm:benign_comp}
The expected prediction error of \npc (\cnpc), measured by KL divergence, is upper-bounded by the sum of the expected prediction error of the attribute predictor and that of the \pc (causal \pc). Specifically,
\begin{align*}
    \bE_{X \sim \pr_*} \left[ \KL\left( \pr_*(Y\mid X) \|\ \pr_{\theta,w}(Y\mid X) \right) \right]
    &\le
    \bE_{X \sim \pr_*} \left[ \KL\left( \pr_*(A_{1:K} \mid X) \|\ \pr_{\theta}(A_{1:K} \mid X) \right) \right] \\
    &+ \bE_{A_{1:K}\sim\pr_*} \left[ \KL\left( \pr_*(Y\mid A_{1:K}) \|\ \pr_{w}(Y\mid A_{1:K}) \right) \right].
\end{align*}
Equality holds if for each $(x,y)$, there exists a constant $c(x,y)$ such that
$\frac{\pr_*(a_{1:K}\mid x)\pr_*(y\mid a_{1:K})}{\pr_{\theta}(a_{1:K}\mid x)\pr_{w}(y\mid a_{1:K})} = c(x,y)$ for all $a_{1:K}$.
\end{theorem}

\begin{proof}

We proceed in three steps.

\paragraph{Step 1.} Show that $$\bE_{X \sim \pr_*} \left[ \KL\left( \pr_*(Y\mid X) \|\ \pr_{\theta,w}(Y\mid X) \right) \right] = \KL\left( \pr_*(X,Y) \|\ \pr_{\theta,w}(X,Y) \right),$$ where $\pr_{\theta,w}(X,Y) = \pr_*(X) \ \pr_{\theta,w}(Y\mid X)$.

Rewriting $\pr_*(X,Y)$ as $\pr_*(X) \ \pr_*(Y\mid X)$, we can expand the RHS as follows.
\begin{align*}
    \KL\left( \pr_*(X,Y) \|\ \pr_{\theta,w}(X,Y) \right)
    &= \sum_{x,y} \pr_*(x,y)\ \log\frac{\pr_*(y\mid x)}{\pr_{\theta,w}(y\mid x)} \\
    &= \sum_x \pr_*(x) \sum_y \pr_*(y\mid x)\ \log\frac{\pr_*(y\mid x)}{\pr_{\theta,w}(y\mid x)} \\
    &= \sum_x \pr_*(x)\ \KL\left( \pr_*(Y\mid x) \|\ \pr_{\theta,w}(Y\mid x) \right) \\
    &=\bE_{X\sim \pr_*} \left[ \KL\left( \pr_*(Y\mid X) \|\ \pr_{\theta,w}(Y\mid X) \right) \right].
\end{align*}

\paragraph{Step 2.} Show that $$\KL\left( \pr_*(X,Y) \|\ \pr_{\theta,w}(X,Y) \right) \le \KL\left( \pr_*(X,A_{1:K},Y) \|\ \pr_{\theta,w}(X,A_{1:K},Y) \right),$$ where $\pr_{\theta,w}(X,A_{1:K},Y) = \pr_*(X)\ \pr_{\theta}(A_{1:K}\mid X)\ \pr_{w}(Y\mid A_{1:K})$.

Under Assumption~\ref{assump:sufficient}, $\pr_{\theta,w}(x,a_{1:K},y)$ can be written as $\pr_*(x)\ \pr_{\theta,w}(a_{1:K},y\mid x)$; thus, the marginal $\sum_{a_{1:K}} \pr_{\theta,w}(x,a_{1:K},y)$ equals $\pr_{\theta,w}(x,y)$.

Following this, we can expand the LHS as follows.
\begin{align*}
    \KL\left( \pr_*(X,Y) \|\ \pr_{\theta,w}(X,Y) \right)
    &= \sum_{x,y} \pr_*(x,y)\ \log\frac{\pr_*(x,y)}{\pr_{\theta,w}(x,y)} \\
    &= \sum_{x,y} \left( \sum_{a_{1:K}}\pr_*(x,a_{1:K},y) \right)\ \log\frac{\sum_{a_{1:K}}\pr_*(x,a_{1:K},y)}{\sum_{a_{1:K}}\pr_{\theta,w}(x,a_{1:K},y)} \\
    \text{using the log-sum inequality}
    &\le \sum_{x,y} \sum_{a_{1:K}} \pr_*(x,a_{1:K},y)\ \log \frac{\pr_*(x,a_{1:K},y)}{\pr_{\theta,w}(x,a_{1:K},y)} \\
    &= \KL\left( \pr_*(X,A_{1:K},Y) \|\ \pr_{\theta,w}(X,A_{1:K},Y) \right).
\end{align*}

Equality holds iff for each $(x,y)$, $\frac{\pr_*(x,a_{1:K},y)}{\pr_{\theta,w}(x,a_{1:K},y)}$ is some constant $c(x,y)$ for all $a_{1:K}$ with $\pr_*(x,a_{1:K},y),\ \pr_{\theta,w}(x,a_{1:K},y)>0$.

\paragraph{Step 3.} Decompose $\KL\left( \pr_*(X,A_{1:K},Y) \|\ \pr_{\theta,w}(X,A_{1:K},Y) \right)$.

Under Assumption~\ref{assump:sufficient}, $\pr_*(X,A_{1:K},Y)$ can be written as $ \pr_*(X)\ \pr_*(A_{1:K}\mid X)\ \pr_*(Y\mid A_{1:K})$.

Following this, we can decompose the target term as follows.
\begin{equation*}
\resizebox{\linewidth}{!}{$
\begin{aligned}
\KL\left( \pr_*(X,A_{1:K},Y) \|\ \pr_{\theta,w}(X,A_{1:K},Y) \right)
    &= \sum_{x,a_{1:K},y} \pr_*(x,a_{1:K},y)\ \log \frac{\pr_*(x,a_{1:K},y)}{\pr_{\theta,w}(x,a_{1:K},y)} \\
    &= \sum_{x,a_{1:K},y} \pr_*(x,a_{1:K},y)\ \log \frac{\pr_*(a_{1:K}\mid x)}{\pr_{\theta}(a_{1:K}\mid x)}
    + \sum_{x,a_{1:K},y} \pr_*(x,a_{1:K},y)\ \log \frac{\pr_*(y\mid a_{1:K})}{\pr_{w}(y\mid a_{1:K})}.
\end{aligned}
$}
\end{equation*}
We use $(*_1)$ and $(*_2)$ to denote these two terms.

We rewrite the first term.
\begin{align*}
    (*_1)
    &= \sum_{x,a_{1:K}} \pr_*(x,a_{1:K})\ \log \frac{\pr_*(a_{1:K}\mid x)}{\pr_{\theta}(a_{1:K}\mid x)} \\
    &= \sum_{x} \pr_*(x) \sum_{a_{1:K}} \pr_*(a_{1:K}\mid x) \log \frac{\pr_*(a_{1:K}\mid x)}{\pr_{\theta}(a_{1:K}\mid x)} \\
    &= \sum_{x} \pr_*(x)\ \KL\left( \pr_*(A_{1:K}\mid x) \|\ \pr_{\theta}(A_{1:K}\mid x) \right) \\
    &= \bE_{X\sim\pr_*} \left[ \KL\left( \pr_*(A_{1:K}\mid X) \|\ \pr_{\theta}(A_{1:K}\mid X) \right) \right].
\end{align*}

We also rewrite the second term.
\begin{align*}
    (*_2)
    &= \sum_{x,a_{1:K}} \pr_*(x,a_{1:K}) \sum_{y} \pr_*(y\mid a_{1:K})\ \log \frac{\pr_*(y\mid a_{1:K})}{\pr_{w}(y\mid a_{1:K})} \\
    &= \sum_{x,a_{1:K}} \pr_*(x,a_{1:K})\ \KL\left( \pr_*(Y\mid a_{1:K}) \|\ \pr_{w}(Y\mid a_{1:K}) \right) \\
    &= \sum_{a_{1:K}} \pr_*(a_{1:K})\ \KL\left( \pr_*(Y\mid a_{1:K}) \|\ \pr_{w}(Y\mid a_{1:K}) \right) \\
    &= \bE_{A_{1:K}\sim\pr_*} \left[ \KL\left( \pr_*(Y\mid A_{1:K}) \|\ \pr_{w}(Y\mid A_{1:K}) \right) \right].
\end{align*}

Summarizing the three steps, we have:
\begin{align*}
    \bE_{X \sim \pr_*} \left[ \KL\left( \pr_*(Y\mid X) \|\ \pr_{\theta,w}(Y\mid X) \right) \right]
    &\le
    \bE_{X \sim \pr_*} \left[ \KL\left( \pr_*(A_{1:K} \mid X) \|\ \pr_{\theta}(A_{1:K} \mid X) \right) \right] \\
    &+ \bE_{A_{1:K}\sim\pr_*} \left[ \KL\left( \pr_*(Y\mid A_{1:K}) \|\ \pr_{w}(Y\mid A_{1:K}) \right) \right].
\end{align*}
Equality holds iff for each $(x,y)$, $\frac{\pr_*(x,a_{1:K},y)}{\pr_{\theta,w}(x,a_{1:K},y)} = \frac{\pr_*(a_{1:K}\mid x)\pr_*(y\mid a_{1:K})}{\pr_{\theta}(a_{1:K}\mid x)\pr_{w}(y\mid a_{1:K})}$ is some constant $c(x,y)$ for all $a_{1:K}$ with $\pr_*(x,a_{1:K},y),\ \pr_{\theta,w}(x,a_{1:K},y)>0$.

\end{proof}

\begin{corollary}[Restatement of Corollary~\ref{thm:int_comp_npc}]
The expected interventional error of \npc is upper-bounded by the sum of the expected interventional error of the attribute predictor and the expected prediction error of the \pc. Specifically,
{\small
\begin{align*}
    \bE_{X\sim\pr_*^{\doint_j}} \left[ \KL\left( \pr_*^{\doint_j}(Y\mid X) \|\ \pr^{\doint_j}_{\theta,w}(Y\mid X) \right) \right]
    &\le
    \bE_{X\sim\pr_*^{\doint_j}} \left[ \KL\left( \pr_*^{\doint_j}(A_{1:K} \mid X) \|\ \pr^{\doint_j}_{\theta}(A_{1:K} \mid X) \right) \right] \\
    &+ \bE_{A_{1:K}\sim\pr_*^{\doint_j}} \left[ \KL\left( \pr_*(Y\mid A_{1:K}) \|\ \pr_{w}(Y\mid A_{1:K}) \right) \right].
\end{align*}}
Equality holds if for each $(x,y)$, there exists a constant $c(x,y)$ such that
$\frac{\pr_*^{\doint_j}(a_{1:K}\mid x)\pr_*(y\mid a_{1:K})}{\pr_{\theta}^{\doint_j}(a_{1:K}\mid x)\pr_{w}(y\mid a_{1:K})} = c(x,y)$ for all $a_{1:K}$.
\end{corollary}

\begin{proof}
This result is a corollary of \Cref{thm:benign_comp}. The proof follows the same steps, except that all distributions are taken under the intervention. For completeness, we still spell out the three steps below.

\paragraph{Step 1.} Show that $$\bE_{X \sim \pr_*^{\doint_j}} \left[ \KL\left( \pr_*^{\doint_j}(Y\mid X) \|\ \pr_{\theta,w}^{\doint_j}(Y\mid X) \right) \right] = \KL\left( \pr_*^{\doint_j}(X,Y) \|\ \pr_{\theta,w}^{\doint_j}(X,Y) \right),$$ where $\pr_{\theta,w}^{\doint_j}(X,Y) = \pr_*^{\doint_j}(X) \ \pr_{\theta,w}^{\doint_j}(Y\mid X)$.

Rewriting $\pr_*^{\doint_j}(X,Y)$ as $\pr_*^{\doint_j}(X) \ \pr_*^{\doint_j}(Y\mid X)$, we can expand the RHS as follows.
\begin{align*}
    \KL\left( \pr_*^{\doint_j}(X,Y) \|\ \pr_{\theta,w}^{\doint_j}(X,Y) \right)
    &= \sum_{x,y} \pr_*^{\doint_j}(x,y)\ \log\frac{\pr_*^{\doint_j}(y\mid x)}{\pr_{\theta,w}^{\doint_j}(y\mid x)} \\
    &= \sum_x \pr_*^{\doint_j}(x) \sum_y \pr_*^{\doint_j}(y\mid x)\ \log\frac{\pr_*^{\doint_j}(y\mid x)}{\pr_{\theta,w}^{\doint_j}(y\mid x)} \\
    &= \sum_x \pr_*^{\doint_j}(x)\ \KL\left( \pr_*^{\doint_j}(Y\mid x) \|\ \pr_{\theta,w}^{\doint_j}(Y\mid x) \right) \\
    &=\bE_{X\sim \pr_*^{\doint_j}} \left[ \KL\left( \pr_*^{\doint_j}(Y\mid X) \|\ \pr_{\theta,w}^{\doint_j}(Y\mid X) \right) \right].
\end{align*}

\paragraph{Step 2.} Show that $$\KL\left( \pr_*^{\doint_j}(X,Y) \|\ \pr_{\theta,w}^{\doint_j}(X,Y) \right) \le \KL\left( \pr_*^{\doint_j}(X,A_{1:K},Y) \|\ \pr_{\theta,w}^{\doint_j}(X,A_{1:K},Y) \right),$$ where $\pr_{\theta,w}^{\doint_j}(X,A_{1:K},Y) = \pr_*^{\doint_j}(X)\ \pr_{\theta}^{\doint_j}(A_{1:K}\mid X)\ \pr_{w}^{\doint_j}(Y\mid A_{1:K})$.

Under Assumption~\ref{assump:sufficient}, $\pr_{\theta,w}^{\doint_j}(x,a_{1:K},y)$ can be written as $\pr_*^{\doint_j}(x)\ \pr_{\theta,w}^{\doint_j}(a_{1:K},y\mid x)$; thus, the marginal $\sum_{a_{1:K}} \pr_{\theta,w}^{\doint_j}(x,a_{1:K},y)$ equals $\pr_{\theta,w}^{\doint_j}(x,y)$.

Following this, we can expand the LHS as follows.
\begin{align*}
    \KL\left( \pr_*^{\doint_j}(X,Y) \|\ \pr_{\theta,w}^{\doint_j}(X,Y) \right)
    &= \sum_{x,y} \pr_*^{\doint_j}(x,y)\ \log\frac{\pr_*^{\doint_j}(x,y)}{\pr_{\theta,w}^{\doint_j}(x,y)} \\
    &= \sum_{x,y} \left( \sum_{a_{1:K}}\pr_*^{\doint_j}(x,a_{1:K},y) \right)\ \log\frac{\sum_{a_{1:K}}\pr_*^{\doint_j}(x,a_{1:K},y)}{\sum_{a_{1:K}}\pr_{\theta,w}^{\doint_j}(x,a_{1:K},y)} \\
    \text{using the log-sum inequality}
    &\le \sum_{x,y} \sum_{a_{1:K}} \pr_*^{\doint_j}(x,a_{1:K},y)\ \log \frac{\pr_*^{\doint_j}(x,a_{1:K},y)}{\pr_{\theta,w}^{\doint_j}(x,a_{1:K},y)} \\
    &= \KL\left( \pr_*^{\doint_j}(X,A_{1:K},Y) \|\ \pr_{\theta,w}^{\doint_j}(X,A_{1:K},Y) \right).
\end{align*}

Equality holds iff for each $(x,y)$, $\frac{\pr_*^{\doint_j}(x,a_{1:K},y)}{\pr_{\theta,w}^{\doint_j}(x,a_{1:K},y)}$ is some constant $c(x,y)$ for all $a_{1:K}$ with $\pr_*^{\doint_j}(x,a_{1:K},y),\ \pr_{\theta,w}^{\doint_j}(x,a_{1:K},y)>0$.

\paragraph{Step 3.} Decompose $\KL\left( \pr_*^{\doint_j}(X,A_{1:K},Y) \|\ \pr_{\theta,w}^{\doint_j}(X,A_{1:K},Y) \right)$.

Under Assumption~\ref{assump:sufficient}, $\pr_*^{\doint_j}(X,A_{1:K},Y)$ can be written as $ \pr_*^{\doint_j}(X)\ \pr_*^{\doint_j}(A_{1:K}\mid X)\ \pr_*^{\doint_j}(Y\mid A_{1:K})$.

Following this, we can decompose the target term as follows.
\begin{equation*}
\resizebox{\linewidth}{!}{$
\begin{aligned}
\KL\left( \pr_*^{\doint_j}(X,A_{1:K},Y) \|\ \pr_{\theta,w}^{\doint_j}(X,A_{1:K},Y) \right)
    &= \sum_{x,a_{1:K},y} \pr_*^{\doint_j}(x,a_{1:K},y)\ \log \frac{\pr_*^{\doint_j}(x,a_{1:K},y)}{\pr_{\theta,w}^{\doint_j}(x,a_{1:K},y)} \\
    &= \sum_{x,a_{1:K},y} \pr_*^{\doint_j}(x,a_{1:K},y)\ \log \frac{\pr_*^{\doint_j}(a_{1:K}\mid x)}{\pr_{\theta}^{\doint_j}(a_{1:K}\mid x)}
    + \sum_{x,a_{1:K},y} \pr_*^{\doint_j}(x,a_{1:K},y)\ \log \frac{\pr_*^{\doint_j}(y\mid a_{1:K})}{\pr_{w}^{\doint_j}(y\mid a_{1:K})}
\end{aligned}
$}
\end{equation*}
We use $(*_1)$ and $(*_2)$ to denote these two terms.

We rewrite the first term.
\begin{align*}
    (*_1)
    &= \sum_{x,a_{1:K}} \pr_*^{\doint_j}(x,a_{1:K})\ \log \frac{\pr_*^{\doint_j}(a_{1:K}\mid x)}{\pr_{\theta}^{\doint_j}(a_{1:K}\mid x)} \\
    &= \sum_{x} \pr_*^{\doint_j}(x) \sum_{a_{1:K}} \pr_*^{\doint_j}(a_{1:K}\mid x) \log \frac{\pr_*^{\doint_j}(a_{1:K}\mid x)}{\pr_{\theta}^{\doint_j}(a_{1:K}\mid x)} \\
    &= \sum_{x} \pr_*^{\doint_j}(x)\ \KL\left( \pr_*^{\doint_j}(A_{1:K}\mid x) \|\ \pr_{\theta}^{\doint_j}(A_{1:K}\mid x) \right) \\
    &= \bE_{X\sim\pr_*^{\doint_j}} \left[ \KL\left( \pr_*^{\doint_j}(A_{1:K}\mid X) \|\ \pr_{\theta}^{\doint_j}(A_{1:K}\mid X) \right) \right].
\end{align*}

By assumption, $\mathrm{Pa}_{\mathcal{C}}(Y) \subseteq \{A_{1:K}\} \subseteq \mathrm{ND}_{\mathcal{C}}(Y)$.
Hence, by the local Markov property and the invariance of causal mechanisms, we have:
\begin{align*}
    \pr_*(Y\mid A_{1:K}) &= \pr_*(Y\mid \mathrm{Pa}(Y)), \\
    \pr^{\doint_j}_*(Y\mid A_{1:K}) &= \pr^{\doint_j}_*(Y\mid \mathrm{Pa}(Y)), \\
    \pr_*(Y\mid \mathrm{Pa}(Y)) &= \pr^{\doint_j}_*(Y\mid \mathrm{Pa}(Y)).
\end{align*}

Consequently, $\pr_*(Y\mid A_{1:K}) = \pr^{\doint_j}_*(Y\mid A_{1:K})$, and similarly, $\pr_w(Y\mid A_{1:K}) = \pr^{\doint_j}_w(Y\mid A_{1:K})$. Thus, the second term can be written as:
\begin{align*}
    (*_2)
    &= \sum_{x,a_{1:K}} \pr_*^{\doint_j}(x,a_{1:K}) \sum_{y} \pr_*^{\doint_j}(y\mid a_{1:K})\ \log \frac{\pr_*^{\doint_j}(y\mid a_{1:K})}{\pr_{w}^{\doint_j}(y\mid a_{1:K})} \\
    &= \sum_{x,a_{1:K}} \pr_*^{\doint_j}(x,a_{1:K}) \sum_{y} \pr_*(y\mid a_{1:K})\ \log \frac{\pr_*(y\mid a_{1:K})}{\pr_{w}(y\mid a_{1:K})} \\
    &= \sum_{x,a_{1:K}} \pr_*^{\doint_j}(x,a_{1:K})\ \KL\left( \pr_*(Y\mid a_{1:K}) \|\ \pr_{w}(Y\mid a_{1:K}) \right) \\
    &= \sum_{a_{1:K}} \pr_*^{\doint_j}(a_{1:K})\ \KL\left( \pr_*(Y\mid a_{1:K}) \|\ \pr_{w}(Y\mid a_{1:K}) \right) \\
    &= \bE_{A_{1:K}\sim\pr_*^{\doint_j}} \left[ \KL\left( \pr_*(Y\mid A_{1:K}) \|\ \pr_{w}(Y\mid A_{1:K}) \right) \right]
\end{align*}

Summarizing the three steps, we have:
\begin{equation*}
\resizebox{\linewidth}{!}{
$
\begin{aligned}
    \bE_{X\sim\pr_*^{\doint_j}} \left[ \KL\left( \pr_*^{\doint_j}(Y\mid X) \|\ \pr^{\doint_j}_{\theta,w}(Y\mid X) \right) \right]
    &\le
    \bE_{X\sim\pr_*^{\doint_j}} \left[ \KL\left( \pr_*^{\doint_j}(A_{1:K} \mid X) \|\ \pr^{\doint_j}_{\theta}(A_{1:K} \mid X) \right) \right] \\
    &+ \bE_{A_{1:K}\sim\pr_*^{\doint_j}} \left[ \KL\left( \pr_*(Y\mid A_{1:K}) \|\ \pr_{w}(Y\mid A_{1:K}) \right) \right].
\end{aligned}
$
}
\end{equation*}
Equality holds iff for each $(x,y)$, $\frac{\pr_*^{\doint_j}(x,a_{1:K},y)}{\pr_{\theta,w}^{\doint_j}(x,a_{1:K},y)} = \frac{\pr_*^{\doint_j}(a_{1:K}\mid x)\pr_*^{\doint_j}(y\mid a_{1:K})}{\pr_{\theta}^{\doint_j}(a_{1:K}\mid x)\pr_{w}^{\doint_j}(y\mid a_{1:K})} = \frac{\pr_*^{\doint_j}(a_{1:K}\mid x)\pr_*(y\mid a_{1:K})}{\pr_{\theta}^{\doint_j}(a_{1:K}\mid x)\pr_{w}(y\mid a_{1:K})}$ is some constant $c(x,y)$ for all $a_{1:K}$ with $\pr_*^{\doint_j}(x,a_{1:K},y),\ \pr_{\theta,w}^{\doint_j}(x,a_{1:K},y)>0$.

\end{proof}

\begin{corollary}[Restatement of Corollary~\ref{thm:int_comp_cnpc}]
The expected interventional error of \cnpc is upper-bounded by a weighted sum of (i) the expected interventional error of the attribute predictor, (ii) the expected interventional error of the causal \pc over attributes, and (iii) the expected prediction error of the causal \pc. Specifically,
\begin{equation*}
\resizebox{\linewidth}{!}{
$
\begin{aligned}
\bE_{X\sim\pr_*^{\doint_j}} \!\left[ \KL\!\left( \pr_*^{\doint_j}(Y\mid X)\ \|\ \tilde{\pr}^{\doint_j}_{\theta,w,\alpha}(Y\mid X) \right) \right] 
&\le
(1-\alpha)\ \bE_{X\sim\pr_*^{\doint_j}} \!\left[ \KL \!\left( \pr_*^{\doint_j}(A_{1:K}\mid X)\ \|\ \pr^{\doint_j}_{\theta}(A_{1:K}\mid X) \right) \right] \\
&+ \alpha\ \bE_{X\sim\pr_*^{\doint_j}} \!\left[ \KL \!\left( \pr_*^{\doint_j}(A_{1:K}\mid X)\ \|\ \pr^{\doint_j}_{w}(A_{1:K}) \right) \right] \\
&+ \bE_{A_{1:K}\sim\pr_*^{\doint_j}} \left[ \KL\left( \pr_*(Y\mid A_{1:K}) \|\ \pr_{w}(Y\mid A_{1:K}) \right) \right].
\end{aligned}
$
}
\end{equation*}
Equality holds if for each $(x,y)$, $\pr^{\doint_j}_\theta(a_{1:K} \mid x)=\pr^{\doint_j}_w(a_{1:K})$ for all $a_{1:K}$, and there exists a constant $c(x,y)$ such that $\frac{\pr_*^{\doint_j}(a_{1:K}\mid x)\pr_*(y\mid a_{1:K})}{\pr_{\theta}^{\doint_j}(a_{1:K}\mid x)\pr_{w}(y\mid a_{1:K})} = c(x,y)$ for all $a_{1:K}$.
\end{corollary}

\begin{proof}
Following the three steps outlined in the proof of Corollary~\ref{thm:int_comp_npc}, we obtain:
\begin{equation*}
\resizebox{\linewidth}{!}{
$
\begin{aligned}
    \bE_{X\sim\pr_*^{\doint_j}} \left[ \KL\left( \pr_*^{\doint_j}(Y\mid X) \|\ \tilde{\pr}^{\doint_j}_{\theta,w,\alpha}(Y\mid X) \right) \right]
    &\le
    \bE_{X\sim\pr_*^{\doint_j}} \left[ \KL\left( \pr_*^{\doint_j}(A_{1:K} \mid X) \|\ \tilde{\pr}^{\doint_j}_{\theta,w,\alpha}(A_{1:K} \mid X) \right) \right] \\
    &+ \bE_{A_{1:K}\sim\pr_*^{\doint_j}} \left[ \KL\left( \pr_*(Y\mid A_{1:K}) \|\ \pr_{w}(Y\mid A_{1:K}) \right) \right].
\end{aligned}
$
}
\end{equation*}
Equality holds iff for each $(x,y)$, $\frac{\pr_*^{\doint_j}(x,a_{1:K},y)}{\tilde{\pr}_{\theta,w,\alpha}^{\doint_j}(x,a_{1:K},y)}$ is some constant $c(x,y)$ for all $a_{1:K}$ with $\pr_*^{\doint_j}(x,a_{1:K},y),\ \tilde{\pr}_{\theta,w,\alpha}^{\doint_j}(x,a_{1:K},y)>0$.

Here, we apply an additional step to decompose $\bE_{X\sim\pr_*^{\doint_j}} \left[ \KL\left( \pr_*^{\doint_j}(A_{1:K} \mid X) \|\ \tilde{\pr}^{\doint_j}_{\theta,w,\alpha}(A_{1:K} \mid X) \right) \right]$.

\begin{align*}
    \KL\left( \pr_*^{\doint_j}(A_{1:K} \mid X) \|\ \tilde{\pr}^{\doint_j}_{\theta,w,\alpha}(A_{1:K} \mid X) \right)
    &= (1-\alpha)\ \KL\left( \pr^{\doint_j}_*( A_{1:K} \mid X ) \|\ \pr^{\doint_j}_\theta( A_{1:K} \mid X ) \right) \\
    &+ \alpha\ \KL\left( \pr^{\doint_j}_*( A_{1:K}\mid X) \|\ \pr^{\doint_j}_w( A_{1:K}) \right) \\
    &+ \log Z_\alpha(X)
\end{align*}

By Hölder’s inequality, $\log Z_\alpha(X)$ can be written as:
\begin{align*}
    \log Z_\alpha(X)
    &= \log \sum_{a_{1:K}} \left( \pr^{\doint_j}_\theta( A_{1:K}=a_{1:K} \mid X ) \right)^{1-\alpha} \left( \pr^{\doint_j}_w( A_{1:K}=a_{1:K}) \right)^{\alpha} \\
    &\le \log \left( \sum_{a_{1:K}}\pr^{\doint_j}_\theta( A_{1:K}=a_{1:K} \mid X ) \right)^{1-\alpha}  \left( \sum_{a_{1:K}}\pr^{\doint_j}_w( A_{1:K}=a_{1:K}) \right)^\alpha \\
    &= 0
\end{align*}
Equality holds when $\frac{\pr^{\doint_j}_\theta( A_{1:K}=a_{1:K} \mid X )}{\pr^{\doint_j}_w( A_{1:K}=a_{1:K})} = 1$ for all $a_{1:K}$.

Summarizing all steps, we have:
\begin{equation*}
\resizebox{\linewidth}{!}{
$
\begin{aligned}
    \bE_{X\sim\pr_*^{\doint_j}} \!\left[ \KL\!\left( \pr_*^{\doint_j}(Y\mid X)\ \|\ \tilde{\pr}^{\doint_j}_{\theta,w,\alpha}(Y\mid X) \right) \right]
    &\le
    (1-\alpha)\ \bE_{X\sim\pr_*^{\doint_j}} \!\left[ \KL \!\left( \pr_*^{\doint_j}(A_{1:K}\mid X)\ \|\ \pr^{\doint_j}_{\theta}(A_{1:K}\mid X) \right) \right] \\
    &+ \alpha\ \bE_{X\sim\pr_*^{\doint_j}} \!\left[ \KL \!\left( \pr_*^{\doint_j}(A_{1:K}\mid X)\ \|\ \pr^{\doint_j}_{w}(A_{1:K}) \right) \right] \\
    &+ \bE_{A_{1:K}\sim\pr_*^{\doint_j}} \left[ \KL\left( \pr_*(Y\mid A_{1:K}) \|\ \pr_{w}(Y\mid A_{1:K}) \right) \right].
\end{aligned}
$
}
\end{equation*}

Equality holds iff for each $(x,y)$, $\pr^{\doint_j}_\theta( a_{1:K} \mid x ) = \pr^{\doint_j}_w( a_{1:K})$ for all $a_{1:K}$, and $\frac{\pr_*^{\doint_j}(x,a_{1:K},y)}{\tilde{\pr}_{\theta,w,\alpha}^{\doint_j}(x,a_{1:K},y)} = \frac{\pr_*^{\doint_j}(a_{1:K}\mid x)\pr_*^{\doint_j}(y\mid a_{1:K})}{\tilde{\pr}_{\theta,w,\alpha}^{\doint_j}(a_{1:K}\mid x)\pr_{w}^{\doint_j}(y\mid a_{1:K})} = \frac{\pr_*^{\doint_j}(a_{1:K}\mid x)\pr_*^{\doint_j}(y\mid a_{1:K})}{\pr_{\theta}^{\doint_j}(a_{1:K}\mid x)\pr_{w}^{\doint_j}(y\mid a_{1:K})} = \frac{\pr_*^{\doint_j}(a_{1:K}\mid x)\pr_*(y\mid a_{1:K})}{\pr_{\theta}^{\doint_j}(a_{1:K}\mid x)\pr_{w}(y\mid a_{1:K})}$ is some constant $c(x,y)$ for all $a_{1:K}$ with $\pr_*^{\doint_j}(x,a_{1:K},y),\ \tilde{\pr}_{\theta,w,\alpha}^{\doint_j}(x,a_{1:K},y)>0$.

\end{proof}



\end{document}